\documentclass{article}

\usepackage[final,nonatbib]{nips_2018}

\usepackage[utf8]{inputenc} 
\usepackage[T1]{fontenc}    
\usepackage[colorlinks]{hyperref}
\usepackage{url}            
\usepackage{booktabs}       
\usepackage{amsfonts}       
\usepackage{nicefrac}       
\usepackage{microtype}      
\usepackage{graphicx}
\usepackage[normalem]{ulem}

\usepackage{amsmath}
\usepackage{amssymb}
\usepackage{graphicx}
\usepackage{url}
\usepackage[small,bf]{caption}
\usepackage{subcaption}
\usepackage{array}

\usepackage[ruled,vlined]{algorithm2e}
\usepackage[usenames, dvipsnames]{color}
\usepackage{csquotes}
\newtheorem{theorem}{Theorem}
\newtheorem{lemma}{Lemma}
\newtheorem{fact}{Fact}

\newtheorem{remark}{Remark}
\newtheorem{definition}{Definition}

\newtheorem{corollary}{Corollary}
\newcommand{\BlackBox}{\rule{1.5ex}{1.5ex}}  
\newenvironment{proof}{\par\noindent{\bf Proof\ }}{\hfill\BlackBox\\[2mm]}

\newcommand{\field}[1]{\mathbb{#1}}

\newcommand{\E}{\field{E}}

\newcommand{\defeq}{\stackrel{\rm def}{=}}

\newcommand{\degree}{\mathrm{deg}}

\newcommand{\scF}{\mathcal{F}}
\newcommand{\scB}{\mathcal{B}}

\newcommand{\scM}{\mathcal{M}}
\newcommand{\scN}{\mathcal{N}}

\newcommand{\bv}{\boldsymbol{v}}

\newcommand{\bp}{\boldsymbol{p}}
\newcommand{\bd}{\boldsymbol{d}}
\newcommand{\bc}{\boldsymbol{c}}

\newcommand{\bzero}{\boldsymbol{0}}

\newcommand{\scC}{\mathcal{C}}

\newcommand{\ctg}{\mathrm{\textsl{cluster}}}
\newcommand{\scO}{\mathcal{O}}
\newcommand{\scE}{\mathcal{E}}
\newcommand{\scA}{\mathcal{A}}
\newcommand{\scT}{\mathcal{T}}

\newcommand{\scV}{\mathcal{V}}

\renewcommand{\Pr}{\field{P}}

\newcommand{\wt}{\widetilde}
\newcommand{\scP}{\mathcal{P}}

\newcommand{\cV}{V^{\mathsf{c}}}
\usepackage{stackengine}
\def\defeq{\mathrel{\ensurestackMath{\stackon[1pt]{=}
{\scriptscriptstyle\Delta}}}}

\newcommand{\idom}{\textsc{idomm}}
\newcommand{\om}{\textsc{oomm}}
\newcommand{\uromm}{\textsc{uromm}}
\newcommand{\smi}{\textsc{smile}}
\newcommand{\scEB}{\mathcal{E}^B}
\newcommand{\scEG}{\mathcal{E}^G}
\newcommand{\bB}{\boldsymbol{B}}
\newcommand{\bG}{\boldsymbol{G}}
\newcommand{\bM}{\boldsymbol{M}}

\newcommand{\bg}{\boldsymbol{g}}
\definecolor{verylightgray}{rgb}{0.8,0.8,0.8}

\title{Online Reciprocal Recommendation with Theoretical Performance Guarantees
}
\author{	   Fabio Vitale\\
	   Department of Computer Science\\
       Sapienza University of Rome (Italy) \& INRIA Lille (France)\\
       Rome, Italy \& Lille, France\\
 	   \texttt{fabio.vitale@inria.fr}  \\
       \And
       Nikos Parotsidis\\
       Department of Computer Science\\
       University of Rome Tor Vergata\\
       Rome, Italy\\
	   \texttt{nikos.parotsidis@uniroma2.it}
	   \And
		Claudio Gentile\\
       INRIA Lille Nord Europe \& Google (New York, USA)\\
       Lille, France \& New York, USA\\
       \texttt{cla.gentile@gmail.com}\\
}

\begin{document}

\maketitle

\begin{abstract}
%
A reciprocal recommendation problem is one where the goal of learning is not just to predict a user’s preference towards a passive item (e.g., a book), but to recommend the targeted user on one side another user from the other side such that a mutual interest between the two exists. The problem thus is sharply different from the more traditional items-to-users recommendation, since a good match requires meeting the preferences at both sides.
We initiate a rigorous theoretical investigation of the reciprocal recommendation task in a specific framework of sequential learning. We point out general limitations, formulate reasonable assumptions enabling effective learning and, under these assumptions, we design and analyze a computationally efficient algorithm that uncovers mutual likes at a pace comparable to that achieved by a clearvoyant algorithm knowing all user preferences in advance. Finally, we validate our algorithm against synthetic and real-world datasets, showing improved empirical performance over simple baselines.
\end{abstract}

\section{Introduction}
%
Recommendation Systems are at the core of many successful online businesses, from e-commerce, to online streaming, to computational advertising, and beyond. These systems have extensively been investigated by both academic and industrial researchers by following the standard paradigm of items-to-users preference prediction/recommendation. In this standard paradigm, a targeted user is presented with a list of items that s/he may prefer according to a preference profile that the system has learned based on both explicit user features (item data, demographic data, explicitly declared preferences, etc.) and past user activity. In more recent years, due to their hugely increasing interest in the online dating and the job recommendation domains, a special kind of recommendation systems called {\em Reciprocal Recommendation Systems} (RRS) have gained big momentum. The reciprocal recommendation problem is sharply different from the more traditional items-to-users recommendation, {\em since recommendations must satisfy both parties}, i.e., both parties can express their likes and dislikes
and a good match requires meeting the preferences of both. Examples of RRS include, for instance: online recruitment systems (e.g., \textit{LinkedIn}),
\footnote
{
\url{https://www.linkedin.com/}.
}
where a job seeker searches for jobs matching his/her preferences, say salary and expectations, and a recruiter who seeks suitable candidates to fulfil the job requirements; heterosexual online dating systems (e.g., \textit{Tinder}),
\footnote
{
\url{https://tinder.com}.
}
where people have the common goal of finding a partner of the opposite gender; roommate matching systems (e.g., \textit{Badi}),
\footnote
{
\url{https://badiapp.com/en}.
} 
used to connect people looking for a room to those looking for a roommate,
online mentoring systems, customer-to-customer marketplaces, etc.

From a Machine Learning perspective, the main challenge in a RRS is thus to learn {\em reciprocated} preferences, since the goal of the system is not just to
predict a user’s preference towards a passive item (a book, a movie, etc), but to recommend the targeted user on one side another user from the other side such that a mutual interest exists. Importantly enough, the interaction the two involved users have with the system is often {\em staged} and {\em unsynced}. Consider, for instance, a scenario where a user, Geena, is recommended to another user, Bob. The recommendation is successful only if both Geena and Bob mutually agree that the recommendation is good. In the first stage, Bob logs into the system and Geena gets recommended to him; this is like in a standard recommendation system: Bob will give a feedback (say, positive) to the system regarding Geena. Geena may never know that she has been recommended to Bob. In a subsequent stage, some time in the future, also Geena logs in. In an attempt to find a match, the system now recommends Bob to Geena. It is only when also Geena responds positively that the reciprocal recommendation becomes successful.

The problem of reciprocal recommendation has so far being studied mainly in the Data Mining, Recommendation Systems, and Social Network Analysis literature (e.g., \cite{dma10,Akehurst2011,Li2012,kgg12,hzs13,Pizzato2013,Xia2015,Alanazi,mm18}), with some interesting adaptations of standard collaborative filtering approaches to user feature similarity, but it has remained largely unexplored from a theoretical standpoint. Despite each application domain has its own specificity,\footnote
{
For instance, users in an online dating system have relevant visual features, and the system needs specific care in removing popular user bias, i.e., ensuring that popular users are not recommended more often than unpopular ones. 
}
in this paper we abstract such details away, and focus on the broad problem of building matches between the two parties in the reciprocal recommendation problem based on behavioral information only. In particular, we do not consider {\em explicit} user preferences (e.g., those evinced by user profiles), but only the {\em implicit} ones, i.e., those derived from past user behavior. The explicit-vs-implicit user features is a standard dichotomy in Recommendation System practice, and it is by now common knowledge that collaborative effects (aka, implicit features) carry far more information about actual user preferences than explicit features, like, for instance, demographic metadata\cite{pt09}. Similar experimental findings are also reported in the context of RRS in the online dating domain~\cite{Akehurst2012}.

In this paper, we initiate a rigorous theoretical investigation of the reciprocal recommendation problem, and we view it as a sequential learning problem where learning proceeds in a sequence of rounds. At each round, a user from one of the two parties becomes active and, based on past feedback, the learning algorithm (called {\em matchmaker}) is compelled to recommend one user from the other party. The broad goal of the algorithm is to uncover as many mutual interests (called {\em matches}) as possible, and to do so {\em as quickly as possible}. 
We formalize our learning model in Section \ref{s:preliminaries}. After observing that, in the absence of structural assumptions about matches, learning is virtually precluded (Section \ref{s:limitations_and_omni}), we come to consider a reasonable clusterability assumption on the preference of users at both sides. Under these assumptions, we design and analyze a computationally efficient matchmaking algorithm that leverages the correlation across matches. We show that the number of uncovered matches within $T$ rounds is comparable (up to constant factors) to those achieved by an optimal algorithm that knows beforehand all user preferences, provided $T$ and the total number of matches to be uncovered is not too small (Sections \ref{s:limitations_and_omni}, and \ref{s:successfulApproach}). Finally, in Section \ref{s:exp} we present a suite of initial experiments, where we contrast (a version of) our algorithm to noncluster-based random baselines on both synthetic and publicly available real-world benchmarks in the domain of online dating. Our experiments serve the twofold purpuse of validating our structural assumptions on user preferences against real data, and showing the improved matchmaking performance of our algorithm, as compared to simple noncluster-based baselines.

\section{Preliminaries}\label{s:preliminaries}
%
We first introduce our basic notation.
We have a set of users $V$ partitioned into two parties. Though a number of alternative metaphores could be adopted here, for concreteness, we call the two parties $B$ (for ``boys") and $G$ (for ``girls"). Throughout this paper, $g$, $g'$ and $g''$ will be used to denote generic members of $G$, and $b$, $b'$, and $b''$ to denote generic members of $B$. For simplicity, we assume the two parties $B$ and $G$ have the same size $n$. A hidden ground truth about the mutual pre
ferences of the members of the two parties is encoded by a sign function $\sigma\,:\,(B\times G) \cup (G \times B) \rightarrow \{-1,+1\}$. Specifically, for a pairing $(b,g) \in B\times G$, the assignment $\sigma(b,g) = +1$ means that boy $b$ likes girl $g$, and $\sigma(b,g) = -1$ means that boy $b$ dislikes girl $g$. Likewise, given pairing $(g,b) \in G\times B$, we have $\sigma(g,b) = +1$ when girl $g$ likes boy $b$, and $\sigma(g,b) = -1$ when girl $g$ dislikes boy $b$. The ground truth $\sigma$ therefore defines a directed bipartite signed graph collectively denoted as $(\langle B,G\rangle,E,\sigma)$, where $E$, the set of directed edges in this graph, is simply $(B\times G) \cup (G \times B)$, i.e., the sef of all possible $2n^2$ directed egdes in this bipartite graph. A ``+1" edge will sometimes be called a positive edge, while a ``-1" edge will be called a negative edge. Any pair of directed edges $(g,b) \in G\times B$ and $(b,g) \in B\times G$ involving the same two subjects $g$ and $b$ is called a {\em reciprocal} pair of edges. We also say that $(g,b)$ is reciprocal to $(b,g)$, and vice versa. The pairing of signed edges $(g,b)$ and $(b,g)$ is called a {\em match} if and only if $\sigma(b,g) = \sigma(g,b) = +1$. The total number of matches will often be denoted by $M$. See Figure \ref{f:1} for a pictorial illustration.

\begin{figure}[t]
\begin{picture}(-10,150)(-50,160)
\scalebox{0.45}{\includegraphics{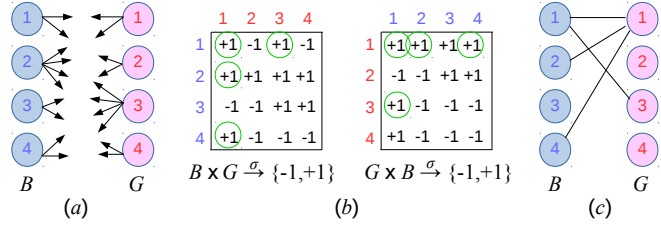}}
\end{picture}
\vspace{-1.3in}
\caption{{\bf (a)} The (complete and directed) bipartite graph $(\langle B,G\rangle,E,\sigma)$ with $n = |B| = |G| = 4$, edges are only sketched. {\bf (b)} Representation of the $\sigma$ function through its two pieces $\sigma\,:\, B\times G \rightarrow \{-1,+1\}$ ($B\times G$ matrix on the left), and $\sigma\,:\, G\times B \rightarrow \{-1,+1\}$ ($G\times B$ matrix on the right). For instance, in this graph, Boy 1 likes Girl 1 and Girl 3, and dislikes Girl 2 and Girl 4, while Girl 3 likes Boy 1, and dislikes Boys 2, 3, and 4. Out of the $n^2= 16$ pairs of reciprocal edges, this graph admits only $M=4$ matches, which are denoted by green circles on both matrices. For instance, the pairing of edges $(1,3)$ and $(3,1)$ are a match since Boy 1 likes Girl 3 and, at the same time, Girl 3 likes Boy 1. {\bf (c)} The associated (undirected and bipartite) matching graph $\scM$. We have, for instance, $\degree_{\scM}({\mbox{Girl 1}}) =3$, and $\degree_{\scM}({\mbox{Boy 2}}) =1$.
\label{f:1}
}
\end{figure}

Coarsely speaking, the goal of a learning algorithm $A$ is to uncover in a sequential fashion as many matches as possible as quickly as possible. More precisely, we are given a time horizon $T \leq n^2$, e.g., $T = n\sqrt{n}$, and at each round $t = 1, \ldots, T$:
\begin{itemize}
\item[(1$_B$)] $A$ receives the id of a boy $b$ chosen uniformly at random\footnote
{
Though different distributional assumptions could be made, for technical simplicity in this paper we decided to focus on the uniform distribution only.
} 
from $B$ ($b$ is meant to be the ``next boy" that logs into the system); 
\item[(2$_B$)] $A$ selects a girl $g' \in G$ to recommend to $b$; 
\item[(3$_B$)] $b$ provides feedback to the learner, in that the sign $\sigma(b,g')$ of the selected boy-to-girl edge is revealed to $A$. 
\end{itemize}
Within the same round $t$, the  three steps described above are subsequently executed after switching the roles of $G$ and $B$ (and will therefore be called Steps (1$_G$), (2$_G$), and (3$_G$)). Hence, each round $t$ is made up of two halves, the first half where a boy at random is logged into the system and the learner $A$ is compelled to select a girl, and the second half where a girl at random is logged in and $A$ has to select a boy. 
Thus at each round $t$, $A$ observes the sign of the two directed edges $(b,g')$ and $(g,b')$, where $b \in B$ and $g \in G$ are generated uniformly at random by the environment, and $g'$ and $b'$ are the outcome of $A$'s recommendation effort. Notice that we assume the ground truth encoded by $\sigma$ is persistent and noiseless, so that whereas the same user (boy or girl) may recur several times throughout the rounds due to their random generation, there is no point for the learner to request the sign of the same edge twice at two different rounds.
The goal of algorithm $A$ is to maximize the number of uncovered matches within the $T$ rounds. The sign of the two reciprocal edges giving rise to a match need not be selected by $A$ in the same round; the round where the match is uncovered is the time when the reciprocating edge is selected, e.g., if in round $t_1$ we observe $\sigma(b,g') = -1$, $\sigma(g,b') = +1$, and in round $t_2 > t_1$ we observe $\sigma(b',g) = +1$, $\sigma(g'',b'') = +1$, we say that the match involving $b'$ and $g$ has been uncovered only in round $t_2$. In fact, if $A$ has uncovered a positive edge $g \rightarrow b'$ in (the second half of) round $t_1$, the reciprocating positive edge $(b',g)$ need not be uncovered any time soon, since $A$ has at the very least to wait until $b'$ will log into the system, an event which on average will occur only $n$ rounds later.

We call {\em matching graph}, and denote it by $\scM$, the bipartite and {\em undirected} graph having $B\cup G$ as nodes, where $(b,g) \in B\times G$ is an edge in $\scM$ if and only if $b$ and $g$ determine a match in the original graph $(\langle B,G\rangle,E,\sigma)$. Given $b \in B$, we let $\scN_{\scM}(b) \subseteq G$ be the set of matching girls for $b$ according to $\sigma$, and $\degree_{\scM}(b)$ be the number of such girls. $\scN_{\scM}(g)$ and $\degree_{\scM}(g)$ are defined symmetrically. See again Figure \ref{f:1} for an example.

The performance of algorithm $A$ is measured by the number of matches found by $A$ within the $T$ rounds. Specifically, if $M_t(A)$ is the number of matches uncovered by $A$ after $t$ rounds of a given run, we would like to obtain {\em lower} bounds on $M_T(A)$ that hold {\em with high probability} over the random generation of boys and girls that log into the system 
as well as the internal randomization of $A$. To this effect, we shall repeatedly use in our statements the acronym {\em w.h.p} to signify with probability at least $1-\scO(\frac{1}{n})$, as $n \rightarrow \infty$. It will also be convenient to denote by $E_t(A)$ the set of directed edges selected by $A$ during the first $t$ rounds, with $E_0(A)=\emptyset$. A given run of $A$ may therefore be summarized by the sequence $\{E_t(A)\}_{t=1}^T$. Likewise, $E^r_t(A)$ will denote the set of reciprocal (not necessarily matching) directed edges selected by $A$ up to time $t$. Finally, $E^r$ will denote the set of {\em all} $|B|\cdot|G|=n^2$ pairs of reciprocal (not necessarily matching) edges between $B$ and $G$.

We will first show (Section \ref{s:limitations_and_omni}) that in the absence of further assumptions on the way the matches are located, there is not much one can do but try and simulate a random sampler. In order to further illustrate our model, the same section introduces a reference optimal behavior that assumes prior knowledge of the whole sign fuction $\sigma$. This will be taken as a yardstick to be contrasted to the performance of our algorithm \smi\, (Section \ref{s:successfulApproach}) that works under more specific, yet reasonable, structural assumptions on $\sigma$.


\section{General Limitations and Optimal Behavior}\label{s:limitations_and_omni}
%
We now show\footnote
{
All proofs are provided in the appendix.
} 
that in the absence of specific assumptions on $\sigma$, the best thing to do in order to uncover matches is to reciprocate at random, no matter how big the number $M$ of matches actually is.
%
\begin{theorem}\label{th:adversarialUB}
Given $B$ and $G$ such that $|B|=|G|=n$, and any integer $m\le \frac{n^2}{2}$, there exists a randomized strategy for generating $\sigma$ such that $M = m$, and the expected number of matches uncovered by {\em any} algorithm $A$ operating on $(\langle B,G\rangle,E,\sigma)$ satisfies\footnote
{
Recall that an {\em upper} bound on $M_T(A)$ is a negative result here, since we are aimed at making $M_T(A)$ as large as possible.
}
\[
\E M_T(A)=\scO\left(\frac{T}{n^2}M\right)~.
\]
\end{theorem}
An algorithm matching the above upper bound is described next. We call this algorithm \om~(Oblivious Online Match Maker), 
The main idea is to develop a strategy that is able to draw uniformly at random as many pairs of reciprocal edges as possible from $E_r$ (recall that $E_r$ is 
the set of {\em all} reciprocal edges between $B$ and $G$). In particular, within the $T$ rounds, \om\, will draw uniformly at random $\Theta(T)$-many such pairs.
The pseudocode of \om\,is given next. For brevity, throughout this paper an algorithm will be described only through Steps ($2_B$) and ($2_G$) -- recall Section~\ref{s:preliminaries}. 

%
\begin{algorithm}
    \SetKwInOut{Input}{\scriptsize{$\triangleright$ INPUT}}
    \SetKwInOut{Step}{$\bullet$ Step}
    	\SetKwBlock{BeginSelectB}{\textbf{$\bullet$ Step 2$_{\bB}$}}{end}
	\SetKwBlock{BeginSelectG}{\textbf{$\bullet$ Step 2$_{\bG}$}}{end}
    \Input{$B$ and $G$}
    {At each round $t$:}\,  (2$_B$) Select $g'$ uniformly at random from $G$\,;\\
                         \hspace{0.97in}(2$_G$) $B_{g,t} \leftarrow \{b''\in B:~~(b'',g) \in E_t(\om),~(g,b'')\not\in E_{t-1}(\om)\}$\;
                     \hspace{1.23in} {\bf If} $B_{g,t} \neq \emptyset$ {\bf then} select $b'$ uniformly at random from $B_{g,t}$\\ 
                                                    \hspace{1.91in}  {\bf else} select $b'$ uniformly at random from $B$~.
    \caption{\om~(Oblivious~Online~Match~Maker)}
\end{algorithm}
%

\om\, simply operates as follows.
In Step ($2_B$) of round $t$, the algorithm chooses a girl $g'$ uniformly at random from the whole set $G$. \om\, maintains over time the set $B_{g,t} \subseteq B$ of all boys that so far gave their feedback (either positive or negative) on $g$, but for whom the feedback from $g$ is not available yet.
In Step ($2_G$), if $B_{g,t}$ is not empty, \om\, chooses a boy uniformly at random from $B_{g,t}$, otherwise it selects a boy uniformly at random from the whole set $B$.\footnote
{
A boy could be selected more than once while serving a girl $g$ during the $T$ rounds. The optimality of \om\ (see Theorems~\ref{th:adversarialUB}~and~\ref{th:adversarialLB}) implies that this redundancy does not significantly affect \om's performance.
} 

Note that, the way it is designed, the selection of $g'$ and $b'$ does {\em not} depend on the signs $\sigma(b,g)$ or $\sigma(g,b)$ collected so far. The following theorem guarantees that $\E M_T(\om)=\Theta\left(\frac{T}{n^2}M\right)$, which is as if we were able to directly sample in most of the $T$ rounds pairs of reciprocal edges.

\begin{theorem}\label{th:adversarialLB}
Given any input graph $(\langle B,G\rangle,E,\sigma)$, with $|B| = |G| = n$, if $T-n=\Omega(n)$ then 
$E^r_T(\om)$ is selected uniformly at random (with replacement) from $E^r$, its size $|E^r_T(\om)|$ is such that
$\E\, |E^r_T(\om)|=\Theta(T)$, and the expected number of matches disclosed by \om\ is such that
\[
\E M_T(\om)
=\Theta\left(\frac{T}{n^2}M\right)~.
\]
\end{theorem}
%
We now describe an optimal behavior (called {\em Omniscient Matchmaker}) that assumes prior knowledge of the whole edge sign assignment $\sigma$. This optimal behavior will be taken as a reference performance for our algorithm of Section~\ref{s:successfulApproach}. This will also help to better clarify our learning model.
%
\begin{definition}
The {\em Omniscient Matchmaker} $A^*$ is an {\em optimal} strategy based on the prior knowledge of the signs $\sigma(b,g)$ and $\sigma(g,b)$ for {\em all} $b\in B$ and $g\in G$. Specifically, based on this information, $A^*$
maximizes the number of matches uncovered during $T$ rounds over {\em all} $n^{2T}$ possible selections that can be made in Steps (2$_B$) and (2$_G$). We denote this optimal number of matches by $M^*_T = M_T(A^*)$. 
\end{definition}

Observe that when the matching graph $\scM$ is such that $\degree_{\scM}(u) > \frac{T}{n}$ for some user $u \in B\cup G$, no algorithm will be able to uncover all $M$ matches in expectation, since Steps (1$_B$) and (1$_G$) of our learning protocol entail that the expected number of times each user $u$ logs into the system is equal to $\frac{T}{n}$. In fact, this holds {\em even} for the Omniscient Matchmaker $A^*$, despite the prior knowledge of $\sigma$.
%
For instance, when $\scM$ turns out to be a random bipartite graph\footnote
{
The matching graph $\scM$ is a random bipartite graph if any edge $(b,g)\in B\times G$ is generated independently with the same probability $p \in [0,1]$.
}
the expected number of matches that any algorithm can achieve is always upper bounded by $\scO\left(\frac{T}{n^2}M\right)$ (this is how Theorem \ref{th:adversarialUB} is proven -- see Appendix \ref{as:proofs}). On the other hand, in order to have $M^*_T=\Theta(M)$ as $n$ grows large, it is sufficient that $\degree_{\scM}(u) \le \frac{T}{n}$ holds for all users $u\in B\cup G$, even with such a random $\scM$.
In order to avoid the pitfalls of $\scM$ being a random bipartite graph (and hence the negative result of Theorem \ref{th:adversarialUB}), we need to slightly depart from our general model of Section \ref{s:preliminaries}, and make structural assumptions on the way matches can be generated. The next section formulates such assumptions, and analyzes an algorithm that under these assumptions is essentially optimal i.t.o. number of uncovered matches. The assumptions and the algorithm itself are then validated against simple baselines on real-world data in the domain of online dating (Section \ref{s:exp}).

\section{A model based on clusterability of received feedback
}\label{s:successfulApproach}
%
In a nutshell, our model is based on the extent to which it is possible to arrange the users in (possibly) {\em overlapping} clusters by means of the feedbacks they may potentially receive from the other party. 
In order to formally describe our model, it will be convenient to introduce the Boolean preference matrices $\bB$, $\bG \in \{0,1\}^{n\times n}$. These two matrices collect in their rows the ground truth contained in $\sigma$, separating the two parties $B$ and $G$. Specifically, $\bB_{i,j} = \frac{1}{2}(1+\sigma(b_i, g_j))$, and $\bG_{i,j} = \frac{1}{2}(1+\sigma(g_i, b_j))$ (these are essentially the matrices exemplified in Figure \ref{f:1}(b) where the \enquote{$-1$} signs therein are replaced by \enquote{$0$}). Then, we consider the $n$ column vectors of $\bB$ (resp. $\bG$) -- i.e., the whole set of feedbacks that each $g \in G$ (resp. $b \in B$) may receive from members of $B$ (resp. $G$) and, for a given radius $\rho \ge 0$, the associated covering number 
of this set of Boolean vectors w.r.t. Hamming distance.
We recall that the covering number at radius $\rho$ is the smallest number of balls of radius $\leq \rho$ that are needed to cover the entire set of $n$ vectors. The smaller $\rho$ the higher the covering number.
If the covering number stays small despite a small $\rho$, then our $n$ vectors can be clustered into a small number of clusters each one having a small (Hamming) radius.

As we mentioned in Section~\ref{s:limitations_and_omni}, a reasonable model for this problem is one for which our learning task can be solved in a nontrival manner, thereby specifically avoiding the pitfalls of $\scM$ being a random bipartite graph. It is therefore worth exploring what pairs of radii and covering numbers may be associated with the two preference matrices $\bG$ and $\bB$ when $\scM$ is indeed random bipartite.
Assume $M=o(n^2)$, so as to avoid pathological cases.
When $\scM$ is random bipartite, one can show that we may have $\rho=\Omega\left(\frac{M}{n}\right)$ even when the two covering numbers 
are both $1$. 
Hence, the only interesting regime is when $\rho = o\left(\frac{M}{n}\right)$. Within this regime, our broad modeling assumption is that the resulting covering numbers for $\bG$ and $\bB$ are $o(n)$, i.e., less that linear in $n$ when $n$ grows large. 
\noindent{\bf Related work. }
The approach of clustering users according to their description/preference similarities while exploiting user feedback is similar in spirit to the two-sided clusterability assumptions investigated, e.g., in \cite{Akehurst2011}, which is based on a mixture of explicit and implicit (collaborative filtering-like) user features.
Yet, as far as we are aware, ours is the first model that lends itself to a rigorous theoretical quantification of matchmaking performance (see Section \ref{ss:smile}).
Moreover, in general in our case the user set is not partitioned as in previous RRS models. Each user may in fact belong to more than one cluster, which is apparently more natural for this problem.  


The reader might also wonder whether the reciprocal recommendation task and associated modeling assumptions share any similarity to the problem of (online) matrix completion/prediction.
Recovering a matrix from a sample of its entries has been widely analyzed by a number of authors with different approaches, viewpoints, and assumptions, e.g., in Statistics and Optimization (e.g., \cite{ct10,klt16}), in Online Learning (e.g., \cite{SSN04,tsu+05,wa07,ha+12,ghp13,Christiano:2014:OLL:2591796.2591880,DBLP:conf/nips/HerbsterPP16}), and beyond. In fact, one may wonder if the problem of predicting the entries of matrices $\bB$ and $\bG$ may somehow be equivalent to the problem of disclosing matches between $B$ and $G$. A closer look reveals that the two tasks are somewhat related, but not quite equivalent, since in reciprocal recommendation the task is to search for matching "ones" between the two binary matrices $\bB$ and $\bG$ by observing entries of the two matrices {\em separately}. In addition, because we get to see at each round the sign of two pairings $(b,g')$ and $(g,b')$, where $b$ and $g$ are drawn at random and $b'$ and $g'$ are selected by the matchmaker, our learning protocol is rather half-stochastic and half-active, which makes the way we gather information about matrix entries quite different from what is usually assumed in the available literature on matrix completion.

\subsection{An Efficient Algorithm}\label{ss:smile}
%
%
%
Under the above modeling assumptions, our goal is to design an efficient matchmaker. We specifically focus on the ability of our algorithm to disclose $\Theta(M)$ matches, in the regime where also the optimal number of matches $M^*_T$ is $\Theta(M)$. 
Recall from Section \ref{s:limitations_and_omni} that the latter assumption is needed so as to 
make the uncovering of $\Theta(M)$ matches possible within the $T$ rounds. Our algorithm, called \smi\ (Sampling Matching Information Leaving out Exceptions) is described as Algorithm \ref{a:smile}. The algorithm depends on input parameter $S\in[\log(n), n/\log(n)]$ and, after randomly shuffling both $B$ and $G$, operates in three phases: Phase 0 (described at the end), Phase I, and Phase II. 

\begin{algorithm}
    \SetKwInOut{Input}{\scriptsize{$\triangleright$ INPUT}}
    \SetKwInOut{Step}{$\bullet$ Step}
    \Input{$B$ and $G$; parameter $S >0$.}
Randomly shuffle sets $B$ and $G$\,;\\ 
Phase 0: Run \om\, to provide an estimate ${\hat M}$ of $M$;\\ 
Phase I: $(\scC,\scF)\gets$ \underline{\textsl{Cluster Estimation}($\langle B,G \rangle$, $S$)}\;
Phase II: \underline{\textsl{User Matching}($\langle B,G \rangle$, $(\scC,\scF)$)}\;
\caption{\smi~\footnotesize{(\textbf{S}ampling \textbf{M}atching \textbf{I}nformation \textbf{L}eaving out \textbf{E}xceptions)}\label{a:smile}}	 
\end{algorithm}

\begin{procedure}[!h]
    \SetKwInOut{Input}{\scriptsize{$\triangleright$ INPUT}}
    \SetKwInOut{Output}{\scriptsize{$\triangleright$ OUTPUT}}
	\SetKwBlock{BeginSelectB}{(2$_B$)}{end}
	\SetKwBlock{BeginSelectG}{(2$_G$)}{end}
    \Input{$B$ and $G$, parameter $S > 0$.}
    \Output{Set of clusters $\scC$, set of feedbacks $\scF$.}
    {\bf Init:} 
    \begin{itemize}
    \vspace{-0.2in}
    \item $F_u \gets \emptyset$\ \ $\forall u \in B\cup G$\,; \texttt{/* One set of feedbacks per user $u \in B\cup G$ */}
    \item $B^r\gets \emptyset$; $G^r\gets \emptyset$\,; \texttt{\hspace{0.24in}/* One set of cluster representatives per side
    */}
    \item $r_{u}\gets 0$\ \ $\forall u \in B\cup G$\,; \texttt{/* No user is candidate to belong to $B^r\cup G^r$ */}
    \end{itemize}
    Let $G=\{g_{1},\ldots, g_{n}\}$,~~~~$B=\{b_{1},\ldots, b_{n}\}$,~~~~$S'\defeq 2S+4\sqrt{S\log n}$, ~~~~ $i,j \gets 1$;\\
    At each round $t$\,:\\
    \eIf{$i\le n~\vee~j\le n$}{
	\BeginSelectB{
	Let $b \in B$ be the boy selected in Step ($1_B$);\\
	\eIf{$i \le n$}{
	Select $g_i$;~~~~$F_{g_i}\gets F_{g_i}\cup\{b\}$\;
	\If{$|F_{g_i}|= S'\wedge r_{g_i}=0$}{\texttt{/* Try to assign $g_i$ to some cluster based on $G^r$ */\\ } 
	\eIf{$\exists g^r \in G^r\, :\, \forall b'\in F_{g_i}\cap F_{g^r}~~s(b',g_i)=s(b',g^r)$}{Set $\ctg(g_i)=g^r$;~~~$i\gets i+1$;
	}{\texttt{/* $g_i$ will be included into $G^r$ as soon as $|F_{g_i}|=\frac{n}{2}$ */\\ } $r_{g_i}\gets 1$\;}
	}
	\texttt{/* If $g_i$ is a cluster representative */}\\
	\If{$|F_{g_i}|=\frac{n}{2}$}{$G^r \gets G^r\cup \{g_i\}$;~~$i\gets i+1$;} 
}{
	Select $g \in G$ arbitrarily;
}
}
	\BeginSelectG{Do the same as in Step (2$_B$) after switching $B$ with $G$, $b$ with $g$, $B_r$ with $G^r$, $i$ with $j$, etc.}
}
{Set:
\begin{itemize}
\item $\ctg(g^r) = g^r~~\forall g^r \in G^r$;
\item $\scC \gets \cup_{u\in B\cup G}\{(u,\ctg(u))\}$;
\item $\scF \gets \cup_{u\in B\cup G}\{(u,F_u)\}$\;
\end{itemize}
\Return{$(\scC,\scF)$~.}}
    \caption{Cluster Estimation() -- \smi~(Phase I) \label{f:phase1}}	 
\end{procedure}

\bigskip

\begin{procedure}[!h]
    \SetKwInOut{Input}{\scriptsize{$\triangleright$ INPUT}}
    \SetKwInOut{Output}{\scriptsize{$\triangleright$ OUTPUT}}
	\SetKwBlock{BeginSelectB}{\textbf{(2$_{\bB}$)}}{end}
	\SetKwBlock{BeginSelectG}{\textbf{(2$_{\bG}$)}}{end}
	\Input{$B$ and $G$, set of clusters $\scC$, set of feedbacks $\scF$.}
	At each round $t$\,:\\
		\BeginSelectB{
		\textit{Let $b \in B$ the boy selected in Step ($1_B$);}\\
		\eIf{$\exists g \in G : b\in F(\ctg(g)) \wedge g\in F(\ctg(b)) \wedge s(b,\ctg(g)) = s(g,\ctg(b))=1 \wedge (b,g)\not\in E_t(\smi)$}{select $g$\;}{
   		select $g\in G$ arbitrarily\;}}
		\BeginSelectG{
		Do the same as in Step (2$_B$) after switching $B$ with $G$, and $b$ with $g$.
	}
    \caption{User Matching() -- \smi~(Phase II) \label{f:phase2}}
\end{procedure}

\textbf{Phase I} (\ref{f:phase1}). \smi\, approximates the clustering over users 
by: i. asking, for each cluster representative $b \in B$, $\Theta(n)$ feedbacks (i.e., edge signs) selected at random from $G$ (and operating symmetrically for each representative $g \in G$), ii. asking $\Theta(S)$-many feedbacks for each remaining user, where parameter $S$ will be set later.
In doing so, \smi\, will be in a position to estimate the clusters each user belongs to, that is, to estimate the matching graph $\scM$, the misprediction per user being w.h.p of the order of $\frac{n\log n}{S}$. The estimated $\scM$ will then be used in Phase II.

A more detailed description of the Cluster Estimation procedure follows (see also pseudocode).
For convenience, we focus on clustering $G$ (hence observing feedbacks from $B$ to $G$), the procedure operates in a completely symmetric way on $B$. Let $F_g$ be the set of all $b \in B$ who provided feedback on $g \in G$ so far.
Assume for the moment we have at our disposal a subset $G^r\subseteq G$ containing one representative for each cluster over $B$, and that for each $g \in G^r$ we have already observed $\frac{n}{2}$ feedbacks provided by $\frac{n}{2}$ {\em distinct} members of $B$, selected {\em uniformly at random} from $B$. Also, let $B(g,S)$ be a subset of $B$ obtained by sampling at random $S'=2S+4\sqrt{S\log n}$-many $b$ from $B$. Then a Chernoff-Hoeffding bound argument shows that for any $g\in G\setminus G^r$ and any $g^r \in G^r$ we have w.h.p. $|B(g,S)\cap F_{g^r}|\ge S$. We use the above to estimate the cluster each $g\in G\setminus G^r$ belongs to. This task can be accomplished by finding $g^r\in G^r$ who receives the same set of feedbacks as that of $g$, i.e., who belongs to the same cluster as $g^r$. 
Yet, in the absence of the feedback provided by {\em all} $b \in B$ to both $g$ and $g^r$, it is not possible to obtain this information with certainty. 
%
The algorithm simply {\em estimates} $g$'s cluster by exploiting Step ($1_B$) of the protocol to ask for feedback on $g$ from $S' = S'(S)$ randomly selected $b \in B$, which will be seen as forming the subset $B(g,S)$. We shall therefore assign $g$ to the cluster represented by an arbitrary $g^r \in G^r$ such that $s(b,g)=s(b,g^r)$ for all $b\in B(g,S)\cap F_{g^r}$. We proceed this way for all $g\in G\setminus G^r$. 

We now remove the assumption on $G^r$.
Although we initially do not have $G^r$, we can build through a concentration argument an approximate version of $G^r$ while asking for the feedback $B(g,S)$ on each unclustered $g$. The Cluster Estimation procedure does so by processing girls $g$ sequentially, as described next. Recall that $G$ was randomly shuffled into an ordered sequence $G =\{ g_1, g_2, \ldots, g_n \}$.
The algorithm maintains an index $i$ over $G$ that only moves forward, and collects feedback information for $g_i$. At any given round, $G^r$ contains all cluster representatives found so far. 
Given $b \in B$ that needs to be served during round $t$ (Step ($1_B$)), we include $b$ in $F_{g_i}$.
If $|F_{g_i}|$ becomes as big as $S'$, then we look for $g\in G^r$ so as to estimate $g_i$'s cluster. 
If we succeed, index $i$ is incremented and the algorithm will collect feedback for $g_i$ during the next rounds. 
If we do not succeed, $g_i$ will be included in $G^r$, and the algorithm will continue to collect feedback on $g_i$ until $|F_{g_i}|<\frac{n}{2}$. When $|F_{g_i}|=\frac{n}{2}$, index $i$ is incremented, so as to consider the next member of $G$.  
Phase I terminates when we have estimated the cluster of each $b$ and $g$ that are themselves not representative of any cluster. 

Finally, when we have concluded with one of the two sides, but not with the other (e.g., we are done with $G$ but not with $B$), we continue with the unterminated side, while for the terminated one we can select members ($g \in G$ in this case) in Step 2 (Step ($2_B$) in this case) arbitrarily.

\textbf{Phase II} (\ref{f:phase2}).
In phase II (see pseudocode), we exploit the feedback collected in Phase I so as to match as many pairs $(b,g)$ as possible. For each user $u \in B\cup G$
selected in Step ($1_B$) or Step ($1_G$), we pick in step ($2_G$) or ($2_B$) a user $u'$ from the other side such that $u'$ belongs to an estimated cluster which is among the set of clusters whose members are liked by $u$, and viceversa.
%
%
When no such $u'$ exists, we select $u'$ from the other side arbitrarily.
%

\textbf{Phase 0: Estimating $M$.}
In the appendix we show that the optimal tuning of $S$ is to set it as a function of the number of hidden matches $M$. Since $M$ is unknown, we run a preliminary phase where we run \om\ (from Section \ref{s:limitations_and_omni}) for a few rounds. Using Theorem~\ref{th:adversarialLB} it is not hard to show that the number $T_{\hat{M}}$ of rounds taken by this preliminary phase to find an estimate $\hat{M}$ of $M$ which is w.h.p. accurate up to a constant factor satisfies $T_{\hat{M}} = \Theta\left(\frac{n^2\log n}{M}\right)$.

In order to quantify the performance of \smi, it will be convenient to refer to the definition of the Boolean preference matrices $\bB$, $\bG \in \{0,1\}^{n\times n}$.
For a given radius $\rho \ge 0$, we denote by $C_{\rho}^G$ the covering number of the $n$ column vectors of $\bB$ w.r.t. Hamming distance. In a similar fashion we define $C_{\rho}^B$.
Moreover, let $C^G$ and $C^B$ be the total number of cluster representatives for girls and boys, respectively, found by \smi, i.e., $C^G=|G^r|$ and $C^B=|B^r|$ at the end of the $T$ rounds.
The following theorem shows that when the optimal number of matches $M^*_T$ is $M$, then so is also $M_T(\smi)$ up to a constant factor, provided $M$ and $T$ are not too small. 
%
\begin{theorem}\label{th:smileMatchMequalMstar}
Given any input graph $(\langle B, G\rangle, E, \sigma)$, with $|B| = |G| = n$, such that $M^*_T=M$ w.h.p. as $n$ grows large, then we have 
\begin{center}
\(
C^G\le\min\Bigl\{\min_{\rho\ge 0}\left(C^G_{\rho/2} +
  3\rho S'
  \right),n\Bigl\}~,
\quad
C^B\le\min\Bigl\{\min_{\rho\ge 0}\left(C^B_{\rho/2} +
  3\rho S'
  \right),n\Bigl\}~.
\)
\end{center}
Furthermore, when $T$ and $M$ are such that
\[
T=\omega\left(n(C^G+C^B+S')\right)
{\mbox{~~~~~~~~~~~and~~~~~~~~~~~~~~}}
M =\omega\bigl(\frac{n^2\log(n)}{S}\bigl)~,
\]
then we have w.h.p.~
\[
M_T(\smi)=\Theta(M)~.
\]
\end{theorem}
Notice in the above theorem the role played by the bounds on
$C^G$ and $C^B$.
If the minimizing $\rho$ therein gives $C^G = C^B = n$, we have enough degrees of freedom for $\scM$ to be generated as a random bipartite graph. 
%
On the other hand, when $C^G$ and $C^B$ are significantly smaller than $n$ at the minimizing $\rho$ (which is what we expect to happen in practice)
the resulting $\scM$ will have a cluster structure that cannot be compatible with a random bipartite graph. This entails that on both sides of the bipartite graph, each subject receives from the other side a set of preferences that can be collectively clustered into a relatively small number of clusters with small intercluster distance.
Then the number of rounds $T$ that \smi\, takes to achieve (up to a constant factor) the same number of matches $M_T^*$ as the Omniscient Matchmaker drops significantly. In particular, when $S$ in \smi\, is set as function of (an estimate of) $M$, we have the following result.
%
\begin{corollary}\label{co:smileMatchMequalMstar}
Given any input graph $(\langle B, G\rangle, E, \sigma)$, with $|B| = |G| = n$, such that $M^*_T=M$ w.h.p. as $n$ grows large, with $T$ and $M$ satisfying
\[
T=\omega\left(n\,(C^G+C^B)+
\frac{n^3\log(n)}{M}\right)~,
\]
where $C^G$ and $C^B$ are upper bounded as in Theorem \ref{th:smileMatchMequalMstar},
then we have w.h.p.~
\[
M_T(\smi)=\Theta(M)~.
\]
\end{corollary}
In order to evaluate in detail the performance of \smi, it is very interesting to show to what extent the conditions bounding from below $T$ in Theorem~\ref{th:smileMatchMequalMstar} are necessary. We have the following general limitation, holding for any matchmaker $A$.
%
\begin{theorem}\label{th:smileTimeNeeded}
Given $B$ and $G$ such that $|B|=|G|=n$, any integer 
$m\in(n\log(n), n^2-n\log(n))$ , and any algorithm $A$ operating on $(\langle B,G\rangle,E,\sigma)$, there exists a randomized strategy for generating $\sigma$ such that $m-\frac{n}{C^G_0+C^B_0-1} < M \le m$, and the number of rounds $T$ needed to achieve $\E M_T(A)=\Theta(M)$, satisfies 
\[
T=\Omega(n\,(C^G_0+C^B_0)+M)~,
\]
as $n \rightarrow \infty$.
\end{theorem}
\begin{remark}
One can verify that the time bound for \smi\ established in Corollary~\ref{co:smileMatchMequalMstar} is nearly optimal whenever $M=\omega\left(n^{3/2}\sqrt{\log(n)}\right)$. To see this, observe that by definition we have
$C^G \le C^G_{0}$ and $C^B \le C^B_{0}$. Now, if $M=\omega\left(n^{3/2}\sqrt{\log(n)}\right)$, then the additive
term $\frac{n^3 \log(n)}{M}$ becomes $o(M)$ and the condition on $T$ in Corollary~\ref{co:smileMatchMequalMstar} 
simply becomes $T=\omega\left(n\,(C^G_0+C^B_0+M')\right)$, where $M'=o(M)$. This has to be contrasted to the lower bound on $T$ contained in Theorem \ref{th:smileTimeNeeded}.

We now explain why it is possible that, when $M=\omega\left(n^{3/2}\sqrt{\log n}\right)$, the additive term
$\frac{n^3\log n}{M}$ in the bound $T=\omega\left(n\,(C^G+C^B)+\frac{n^3\log(n)}{M}\right)$
of Corollary~\ref{co:smileMatchMequalMstar} becomes $o(M)$, while the first term $n\,(C^G+C^B)$ can be upper bounded by 
$n\,(C^G_0+C^B_0)$. Since the lower bound $T=\Omega(n\,(C^G_0+C^B_0)+M)$ of Theorem~\ref{th:smileTimeNeeded} has a linear dependence on $M$,
it might seem quite surprising that the larger $M$ is the smaller becomes the second term in the bound of 
Corollary~\ref{co:smileMatchMequalMstar}.
%
However, it is important to take into account that in Corollary~\ref{co:smileMatchMequalMstar} $T$ must be large enough to satisfy even the condition $M^*_T=M$. Let $T^*$ be the number of rounds $T$ necessary to satisfy w.h.p. $M^*_T=M$. In Corollary~\ref{co:smileMatchMequalMstar}, 
both the conditions $T\ge T^*$ and $T=\omega\left(n\,(C^G+C^B)+\frac{n^3\log(n)}{M}\right)$ must simultaneously hold. When $M$ is large, the number of rounds needed to satisfy the former condition becomes much larger than the one needed for the latter. 

As a further insight, consider the following.
We either have $M=\scO\left(n(C^G+C^B)\right)$ or $M=\omega\left(n(C^G+C^B)\right)$. In the first case, the lower bound in Theorem~\ref{th:smileTimeNeeded} clearly becomes $T=\Omega\left(n\,(C^G_0+C^B_0+C^G+C^B)\right)$, hence not directly depending on $M$. In the second case, whenever $M=\omega\left(n^{3/2}\sqrt{\log(n)}\right)$, 
$T^*$ 
is larger than $n\,(C^G+C^B)+ \frac{n^3\log(n)}{M}$ since, by definition, we must have $T^*=\Omega(M)$, while in this case $n\,(C^G+C^B)+ \frac{n^3\log(n)}{M}=o(M)$.  
In conclusion, if the number of rounds \smi\ takes to uncover $\Theta(M)$ matches equals the number of rounds taken by the Omniscent Matchmaker to uncover exactly $M$ matches, then \smi\ is optimal up to a constant factor, because no algorithm can outperform the Omniscent Matchmaker.
This provides a crucially important insight into the key factors allowing the additive term $\frac{n^3\log n}{M}$ to be equal to $o(M)$ in Corollary~\ref{co:smileMatchMequalMstar}, and is indeed one of the keystones in the proof of Theorem~\ref{th:smileMatchMequalMstar} (see Appendix \ref{as:proofs}).
\end{remark}




We conclude this section by emphasizing the fact that \smi\, is indeed quite scalable.
As proven Appendix \ref{as:proofs}, an implementation of \smi\, exists that leverages a combined use of suitable data-structures, leading to both time and space efficiency.
%
\begin{theorem}\label{th:smileCompComplexity}
The running time of \smi\ is $\scO\left(T+n\,S\,\left(C^G+C^B\right)\right)$, the memory requirement is $\scO(n\,(C^G+C^B))$. Furthermore, when 
\[
T =
\omega\left(n\,(C^G+C^B)+
\frac{n^3\log(n)}{M}\right)~,
\] 
as required by Corollary~\ref{co:smileMatchMequalMstar}, the amortized time per round is
\[
\Theta(1)+o(C^G+C^B)~,
\] 
which is always sublinear in $n$. 
\end{theorem}

\newcommand{\ismi}{\textsc{i-smile}}

\section{Experiments}\label{s:exp}
%
In this section we evaluate the performance of (a variant of) our algorithm by empirically contrasting it to simple baselines against artificial and real-world datasets from the online dating domain. The comparison on real-world data also serve as a validation of our modeling assumptions.
%
%

\begin{table}[t!]
	\centering
	\renewcommand{\arraystretch}{1.2}
\resizebox{\textwidth}{!}{
	\begin{tabular}{|c|cccccccccc|}
		\cline{6-11}
		\multicolumn{5}{c}{}  &
		\multicolumn{6}{|c|}{\#clusters within bounded radius}
		\\ \hline
		\multicolumn{5}{|c}{properties}  &
		\multicolumn{2}{|c}{$2 \cdot n/\log(n)$} & \multicolumn{2}{|c}{$n/\log(n)$} & \multicolumn{2}{|c|}{$0.5 \cdot n/\log(n)$}
		\\ \hline
		\multicolumn{11}{|c|}{Synthetic datasets (2000 boys and 2000 girls)}
		\\ \hline
		& $|\scC(B)|$ & $|\scC(G)|$ & \#likes & \#matches  & \multicolumn{1}{|c}{$|\scC(B)|$} & $|\scC(G)|$ & \multicolumn{1}{|c}{$|\scC(B)|$} & $|\scC(G)|$ & \multicolumn{1}{|c}{$|\scC(B)|$} & $|\scC(G)|$ \\ \hline
		\multicolumn{1}{|l|}{S-20-23} & $20$ & $22$ 
		& $2.45M$
		& $374K$
		 & \multicolumn{1}{|c}{20}&  23 &\multicolumn{1}{|c}{20} & 23 &\multicolumn{1}{|c}{445} & 429
		\\ \hline
		\multicolumn{1}{|l|}{S-95-100} & $95$ & $100$ 
		& $2.46M$
		& $377K$
		 &\multicolumn{1}{|c}{95} & 100 &\multicolumn{1}{|c}{95} & 100 &\multicolumn{1}{|c}{603} & 624
		\\ \hline
		\multicolumn{1}{|l|}{S-500-480} & $500$ & $480$ 
		& $2.47M$
		& $380K$
		 &\multicolumn{1}{|c}{500} & 480 &\multicolumn{1}{|c}{500} & 480 &\multicolumn{1}{|c}{983} &950
		\\ \hline
		\multicolumn{1}{|l|}{S-2000-2000} & $2000$ & $2000$ 
		& $2.47M$
		& $382K$
		 &\multicolumn{1}{|c}{2000} & 2000 &\multicolumn{1}{|c}{2000} & 2000 &\multicolumn{1}{|c}{2000} & 2000
		\\ \hline
		\multicolumn{11}{|c|}{Real-world datasets}  
		\\ \hline
		\multicolumn{1}{|l|}{} & \multicolumn{1}{c}{$|B|$} & \multicolumn{1}{c}{$|G|$} & \multicolumn{1}{c}{\#likes} & \multicolumn{1}{c}{\#matches}    & \multicolumn{1}{|c}{$|\scC(B)|$} & $|\scC(G)|$ & \multicolumn{1}{|c}{$|\scC(B)|$} & $|\scC(G)|$ & \multicolumn{1}{|c}{$|\scC(B)|$} & $|\scC(G)|$ \\ \hline
		RW-1007-1286 & $1007$ & $1286$ 
		& $125K$
		& $13.9K$
		& \multicolumn{1}{|c}{53} & 48 & \multicolumn{1}{|c}{177} & 216 & \multicolumn{1}{|c}{385} & 508
		\\ \hline
		RW-1526-2564 & $1526$ & $2564$ 
		& $227K$
		& $19.6K$
		& \multicolumn{1}{|c}{37} & 45 & \multicolumn{1}{|c}{138} & 216 & \multicolumn{1}{|c}{339} & 601 
		\\ \hline
		RW-2265-3939 & $2265$ & $3939$ 
		& $370K$
		& $25.0K$
		& \multicolumn{1}{|c}{42} & 45 & \multicolumn{1}{|c}{145} & 215 & \multicolumn{1}{|c}{306} & 622
		\\ \hline
	\end{tabular}
	}
	\vspace{.25cm}
	\caption{Relevant properties of our datasets. The last six columns present an approximation to the number of clusters when we allow radius $2 \cdot n/\log(n)$, $n/\log(n)$, and $0.5 \cdot n/ \log(n)$ between users of the same cluster.}
	\label{tab:characteristics}
\end{table}

\paragraph*{Datasets.}
The relevant properties of our datasets are given in Table \ref{tab:characteristics}. Each of our synthetic datasets has $|B| = |G| = 2000$. We randomly partitioned $B$ and $G$ into $C_B$ and $C_G$ clusters, respectively. Each boy likes all the girls of a cluster $C$ with probability $0.2$, and with probability $0.8$ dislikes them. We do the same for the preferences from girls to boy clusters. Finally, for each preference (either positive or negative) we reverse its sign with probability $1/(2\cdot\log n)$ (in our case, $n=2000$). Notice in Table \ref{tab:characteristics} that, for all four datasets we generated, the number of likes is bigger than $|B|\cdot|G|/2$.
%
%
As for real-world datasets, we used the one from \cite{brozovsky07recommender}, which is also publicly available. This is a dataset from a Czech dating website, where 220,970 users rate each other in a scale from 1 (worst) to 10 (best). The gender of the users is not always available.
To get two disjoint parties $B$ and $G$, where each user rates only users from the other party, 
we disregarded all users whose gender is not specified.
As this dataset is very sparse, we extracted dense subsets as follows.
We considered as ''like" any rating $> 2$, while all ratings, including the missing ones, are ''dislikes".
Next, we iteratively removed the users with the smallest number of ratings until we met some desired density level.
Specifically, we executed the above process until we obtained two sets $B$ and $G$ such that the number of likes between the two parties is at least $2(\min\{|B|,|G|\})^{3/2}$ (resulting in dataset RW-1007-1286), $1.75(\min\{|B|,|G|\})^{3/2}$ (dataset RW-1526-2564), or $1.5(\min\{|B|,|G|\})^{3/2}$ (dataset RW-2265-3939).

\paragraph{Random baselines.} We included as baselines \om\,, from Section \ref{s:limitations_and_omni}, and a random method that asks a user for his/her feedback on another user (of opposite gender) picked uniformly at random. We refer to this algorithm as \uromm{}.

\paragraph{Implementation of \smi.}
In the implementation of \smi{}, we slightly deviated from the description in Section \ref{ss:smile}.
One important modification is that we interleaved Phase I and Phase II. The high-level idea is to start exploiting immediately the clusters once some clusters are identified, without waiting to learn all of them. 
Additionally, we gave higher priority to exploring the reciprocal feedback of a discovered like, and we avoided doing so in the case of a dislike. Finally, whenever we test whether two users belong in the same cluster, we allowed a radius of a $(1/\log(n))$ fraction of the tested entries. The parameter $S'$ in \smi\ has been set to $S+\sqrt{S \log n}$. We call the resulting algorithm \ismi{} (Improved \smi). See Appendix \ref{s:exp-app} for more details.

\paragraph*{Evaluation.} 
To get a complete picture on the behavior of the algorithms for different time horizons, we present for each algorithm the number of discovered matches as a function of $T\in \{1,\ldots, 2|B||G|\}$. 
Figure~\ref{fig:plots} contains 3 representative cases, further plots are given in Appendix \ref{s:exp-app}. In all datasets we tested, \ismi{} 
clearly outperforms \uromm{} and \om{}.
Our experiments confirm that \smi{} (and therefore \ismi{}) 
%
%
quickly learns the underlying structure of the likes between users, and uses this structure to reveal 
the matches between them. Moreover, the variant \ismi{} that we implemented allows one not only to perform well on graphs with no underlying structure in the likes, but also to discover matches during the exploration phase while learning the clusters.
A summary of the overall performance of the algorithms is reported in Table \ref{tab:under-curve} in Appendix \ref{s:exp-app}, where we give the \emph{area under the curve} metric, capturing
how quickly, on average, the different algorithms learn over time. Again, \ismi{} is largely outperforming its competitors.


%

\begin{figure}[t!]
\vspace{0.1in}
\vspace{0.01in}
	\begin{center}
		\begin{subfigure}{.32\textwidth}
			\centering
			\includegraphics[trim ={0.5cm 0cm 9.5cm 0}, width=1\linewidth]{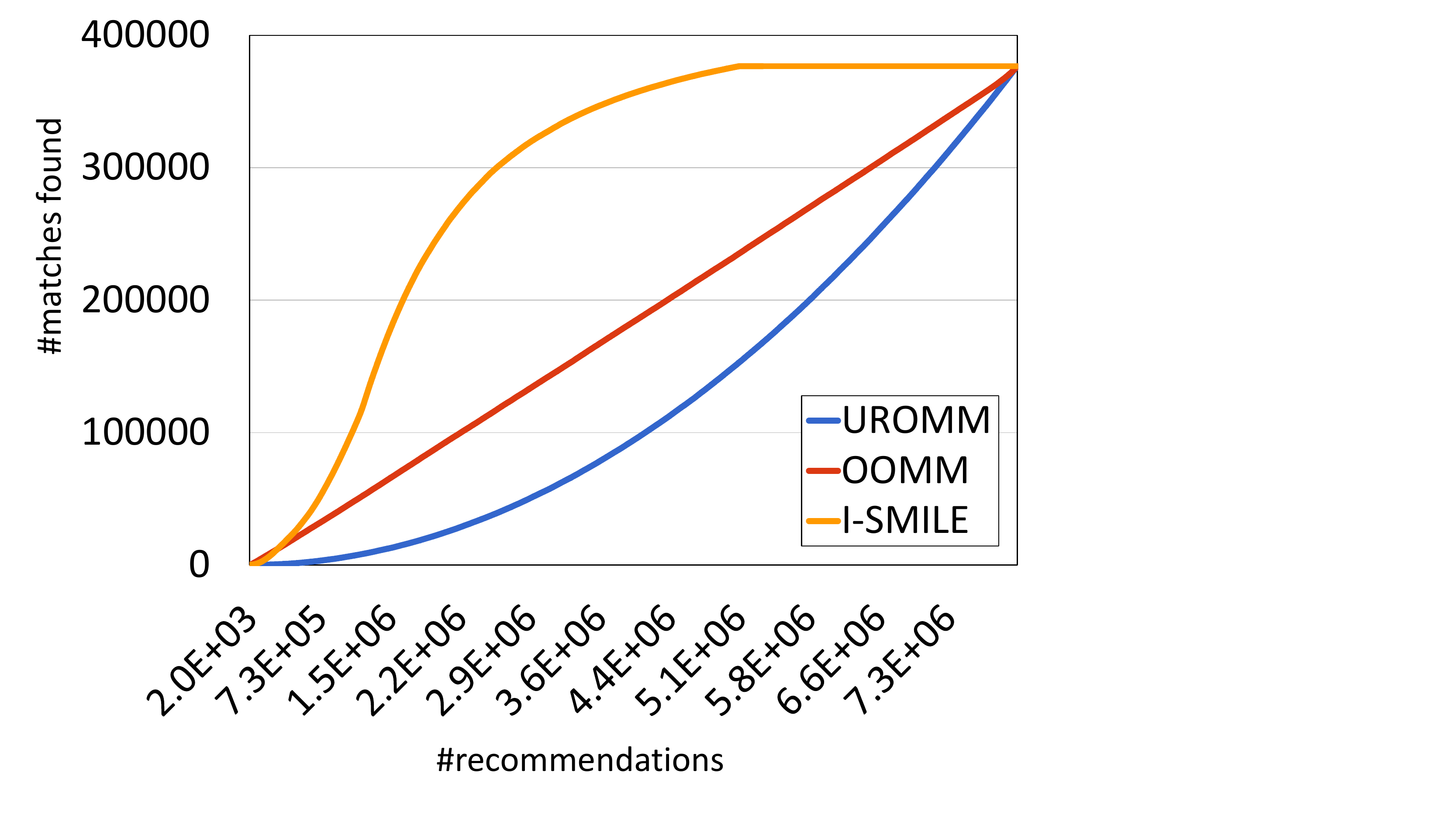}
			\label{fig:sub1}
		\end{subfigure}%
		\begin{subfigure}{.32\textwidth}
			\centering
			\includegraphics[trim ={0.5cm 0cm 9.5cm 0},width=1\linewidth]{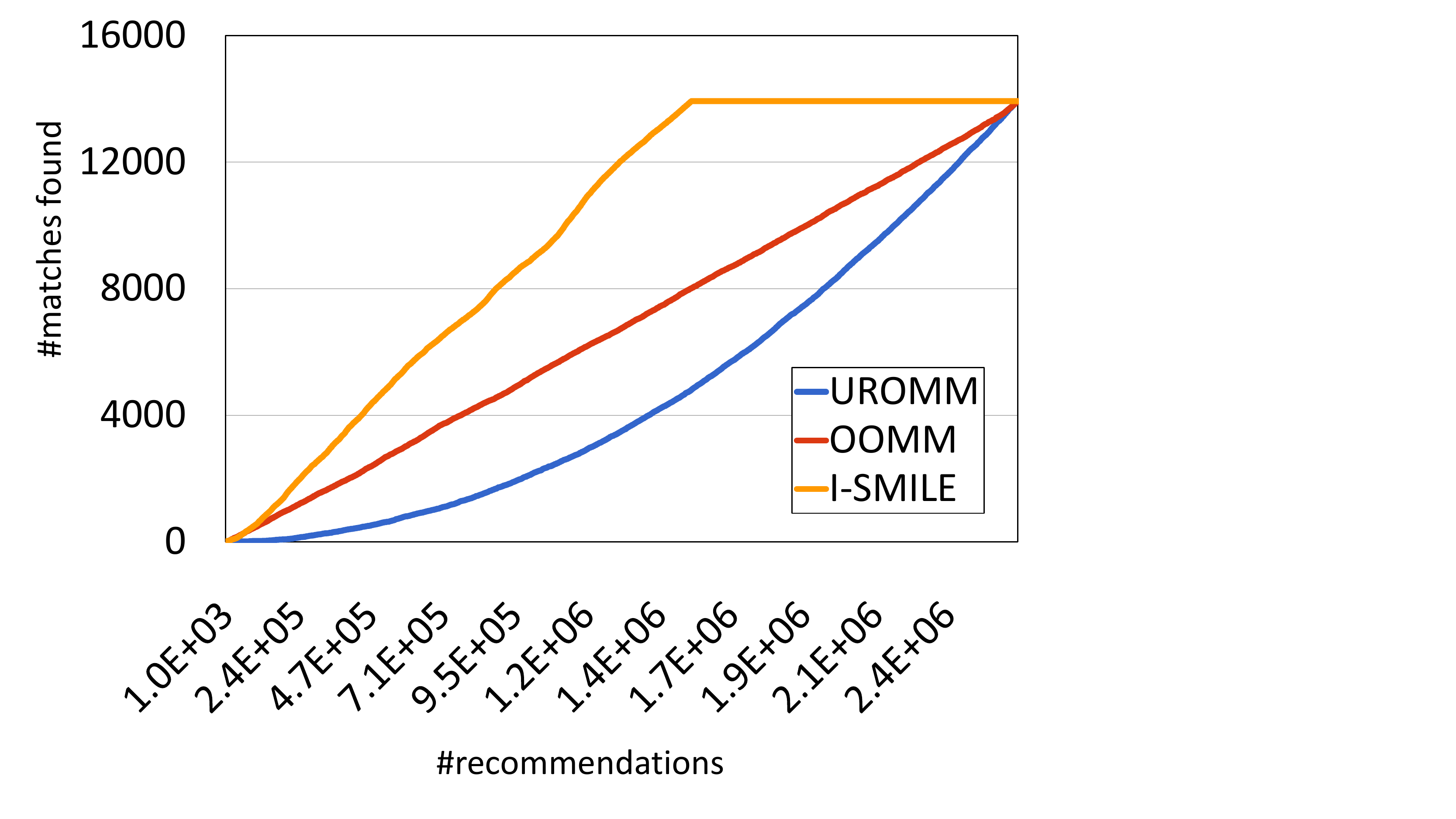}
			\label{fig:sub2}
		\end{subfigure}
		\begin{subfigure}{.32\textwidth}
		\centering
		\includegraphics[trim ={0.5cm 0cm 9.5cm 0},width=1\linewidth]{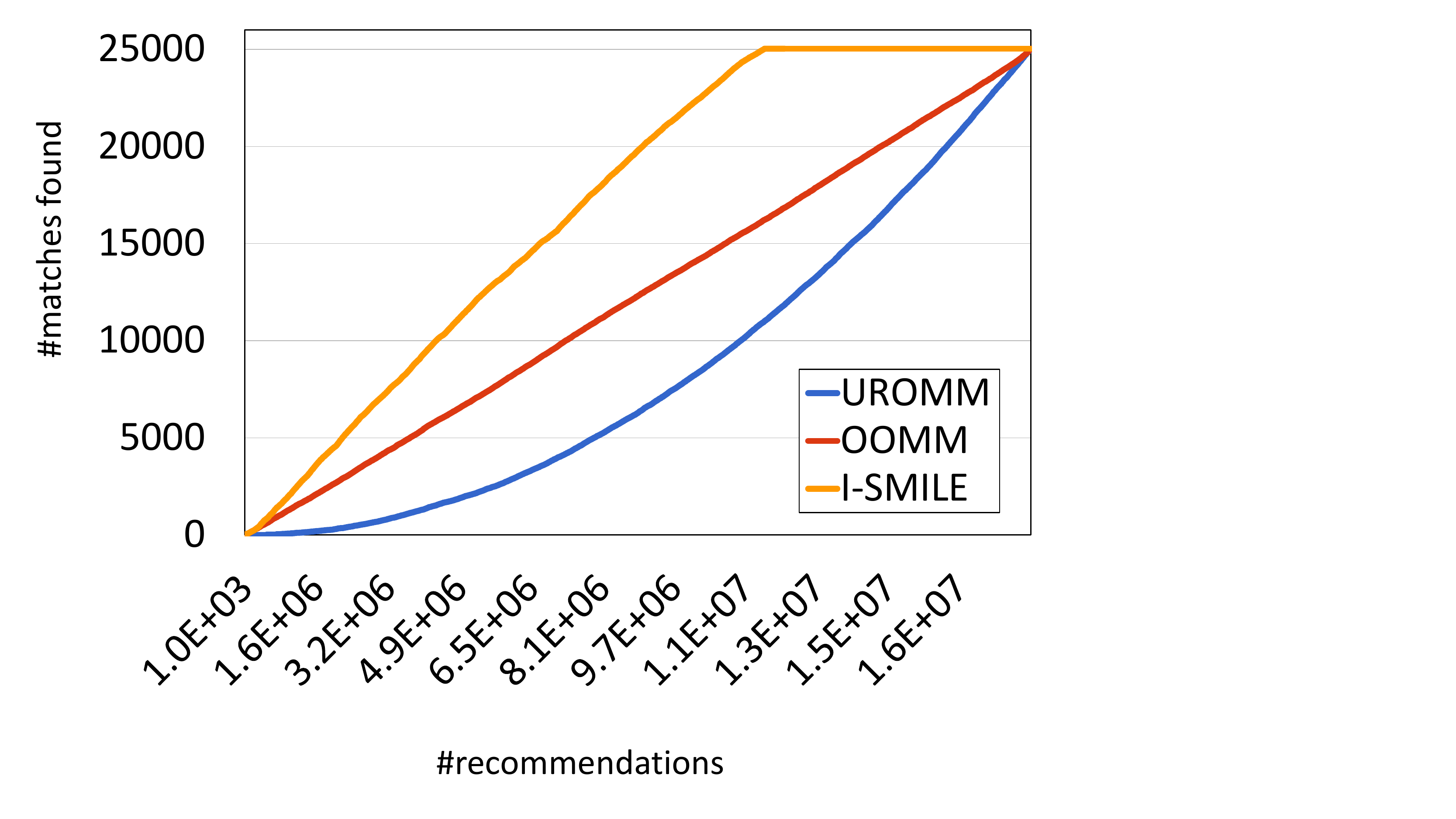}
		\label{fig:sub3}
\end{subfigure}
	\end{center}
	\caption{
	Empirical comparison of the 3 algorithms on datasets S-95-100 (left), RW-1007-1286 (middle), RW-2265-3939 (right). Each plot gives number of disclosed matches vs. time. 
(no of recommendations). \ismi{}'s yellow curve flattens out when there are no more matches to uncover.
}\label{fig:plots}
\end{figure}


\section{Conclusions and Ongoing Research}\label{s:conclusions}
%
We have initiated a theoretical investigation of the problem of reciprocal recommendation in an ad hoc model of sequential learning. 
%
Under suitable clusterability assumptions, 
we have introduced an efficient matchmaker called \smi\,, and have proven its ability to uncover matches at a speed comparable to the Omniscent Matchmaker, so long as $M$ and $T$ are not too small (Theorem \ref{th:smileMatchMequalMstar} and Corollary \ref{co:smileMatchMequalMstar}). Our theoretical findings also include a computational complexity analysis 
(Theorem \ref{th:smileCompComplexity}), as well as limitations on the number of disclosable matches in both the general (Theorem \ref{th:adversarialUB}) and the cluster case (Theorem \ref{th:smileTimeNeeded}).
We complemented our results with an initial set of experiments on synthetic and real-world datasets in the online dating domain, showing encouraging evidence.

Current ongoing research includes:
\begin{itemize}
\item [i.] Introducing suitable noise models for the sign function $\sigma$. 
\item [ii.] Generalizing our learning model to nonbinary feedback preferences.
\item [iii.] Investigating algorithms whose goal is to maximize the area under the curve ``number of matches-vs-time", i.e., the criterion
\(
\sum_{t \in [T]} M_t(A)~,
\)
rather than the one we analyzed in this paper; maximizing this criterion requires interleaving the phases where we collect matches (exploration) and the phases where we do actually disclose them (exploitation).
\item [iv.] More experimental comparisons on different datasets against heuristic approaches available in the literature.
%
\end{itemize}

\appendix
\section{Ancillary Lemmas}\label{as:ancillaryTechnicalities}

\subsection{Hamming distance-based clustering lemmas}\label{s:hammingDistance}
Given an $r\times c$ matrix $A$, an $r$-dimensional vector $\bc$, and a subset of indices $Z \subseteq [r]$, let $\scA(Z, \bc)$ be the set containing all the column vectors $\bv$ of $A$ such that $\bv_i=\bc_i$ for all indices $i \in Z$.
Furthermore, given an integer $k>0$, we denote by $R_k$ a set $k$ distict
integers drawn uniformly at random from $[r]$.
We have the following lemma, whose proof is given in Appendix \ref{le:hammingDistance}.
\begin{lemma}\label{le:hammingDistance}
Given any matrix $A \in \{0, 1\}^{r \times c}$ where $r\ge c>1$, any column vector $\bc$ of $A$, any positive constant $\beta$ and any integer $k\ge \lceil \beta \log r\rceil$, the Hamming distance between $\bc$ and any column vector of $\scA(R_k, \bc)$ is upper bounded by $\frac{\beta r}{k}\log r$ with probability at least
$1-\,r^{\,1-\beta}$.

\end{lemma}
\begin{proof}
Let $R_k=\{i_1, i_2, \ldots, i_k\}$.
Let $\scV(A,\bc)$ be the set of columns vectors $\bv$ of $A$ such that the Hamming distance between $\bc$ and $\bv$ is larger than $\frac{\beta r}{k}\log r$. Clearly, we have $|\scV(A,\bc)|\le c$. Thus, given any vector $\bv \in \scV(A,\bc)$, the probability that it belongs to $\scA(R_k, \boldsymbol{c})$ can be upper bounded as follows:
\begin{align}
\Pr(\bv \in \scA(R_k, \bc))
&=\Pr\left(\bv_{i_j}=\bc_{i_j}~\forall j\in[k]\right)\nonumber\\
&\le
\left(1-\frac{\frac{\beta r}{k}\log r}{r}\right)^{k}\nonumber\\
&=
\left(1-\frac{\beta\log r}{k}\right)^{k}\nonumber\\
&\le r^{-\beta}\nonumber~.
\end{align}
%
%
%
The probability that there exists at least one column vector belonging to both $\scV(A,\bc)$ and $\scA(R_k, \bc)$ can therefore be bounded as follows~:
\begin{align}
\Pr(|\scV(A,\bc) \cap \scA(R_k, \bc)|\neq \emptyset)
&\le\sum_{\bv \in \scV(A,\bc)}\Pr(\bv \in \scA(R_k, \bc))\label{eq:unionBound}\\
&\le|\scV(A,\bc)|
\,r^{-\beta}\nonumber\\
&\le 
c\,r^{-\beta}\nonumber\\
&\le 
r^{1-\beta}\label{eq:setSize}~,
\end{align}
where in Equation~(\ref{eq:unionBound}) we simply use the union bound, and in Equation~(\ref{eq:setSize}) we took into account that $|\scV(A,\bc)|\le c \le r$.
\end{proof}

\subsection{Setting \smi\ parameter $S$}\label{s:settingS}
When putting together the information gathered during phase I, we may both miss to detect pairs of matching users, and consider some pairs of users as part of $\scM$ while they are not. In fact, \smi\, does not completely recover the structure of $\scM$, it only creates an approximate matching graph $\scM'$. Let $E_{\scM}$ and $E_{\scM'}$ be the set of edges of the two matching graphs. The error in reconstructing $\scM$ through $\scM'$ is represented by all edges in $E_{\scM}\triangle E_{\scM'}$, the symmetric difference between $E_{\scM}$ and $E_{\scM'}$. 

During phase I, applying Lemma~\ref{le:hammingDistance} with $\beta=3$, we have that for any user in $B \cup G$, the number of mispredicted feedbacks is w.h.p. bounded by $\frac{3n\log n}{S}$. 
It is not difficult to see that, requesting $\frac{n}{2}$ feedbacks selected uniformly at random for each cluster representative, the number of edges of $\scM$ recovered is w.h.p. equal to $\frac{1}{4}|E_{\scM}|-o(|E_{\scM}|)=\frac{M}{4}-o(M)$,\footnote
{
As we assume in our analysis that $n$ goes to infinity, we also assume that $M$, as a function of $n$, diverges as $n\to\infty$. Note that, even in the lower bound on $M$ contained in the statement of Theorem~\ref{th:smileMatchMequalMstar}, $M$ is in fact always superlinear in $n$ because of the definition of the range of values of $S$, i.e.,  $S\in[\log(n),n/\log(n)]$.
}
Hence, the total number of matches that we do not detect or that we mispredict is upper bounded w.h.p. by 
$\frac{3}{4}M+\frac{6n^2\log n}{S}+o(M)$.

Since our goal is to find w.h.p. $\Theta(M)$ matches (under the assumption that $M^*_T=M$ holds w.h.p.), 
a lower bound on $M$ required to achieve this goal is 
$M \ge \frac{\gamma n^2\log n}{S}$ 
for some constant $\gamma$.
This implies that, by setting $S=\frac{\gamma n^2\log n}{M}$, we are guaranteed to find w.h.p. at least a constant fraction of the total number of matches $M$. 

\section{Proofs}\label{as:proofs}

\subsection{Proof of Theorem~\ref{th:adversarialUB}}

\begin{proof} 
Consider the following adversarial random strategy. We select uniformly at random $m$ elements from the set of pairs 
$B\times G$. For each selected pair $(b,g)$, we set both $\sigma(b,g)$ and $\sigma(g,b)$ to $+1$, and then assign the value $-1$ to all remaining directed edges of $E$. We have therefore $M=m$. 

Given any algorithm $A$, recall that $E_T(A)$ denotes the set of directed edges selected by $A$ during $T$ rounds. We now define $E'_T(A)$ as the following superset of $E_T(A)$\,:
\[
E'_T(A) \defeq E_T(A) 
\cup \{(g',b') : (b',g') \in E_T(A)\}
\cup \{(b'',g'') : (g'',b'') \in E_T(A)\}~.
\]
$E'_T(A)$ contains all directed edges
$(b',g')$ and $(g'',b'')$ already contained in $E_T(A)$
~together with their respective reciprocal edges $(g',b')$ and 
$(b'',g'')$.

Let now $M'_T(A)$ be the number of matches formed by the edges in $E'_T(A)$:
\[
M'_T(A) \defeq |\{~\{b,g\}: (b,g)  ,(g,b)\in E'_T(A),~\sigma(b,g)  = \sigma(g,b)=+1\}|~.
\]
By the definition of $M'_T(A)$, we know that $M'_T(A)\ge M_T(A)$
and $|E'_T(A)|\le 2|E_T(A)|$ which in turn is equal to $4T$, because during each round two distinct edges are selected. The number of pairs of reciprocal edges of $|E'_T(A)|$ is $\frac{|E'_T(A)|}{2}\le2T$, because for each edge $(u,u') \in E'_T(A)$, we always have $(u',u) \in E'_T(A)$.
Furthermore, because of the randomized sign assignment strategy described above, for any pair of reciprocal edges in $E'_T(A)$, the probability that this pair is a match is equal to $\frac{M}{n^2}$ {\em irrespective} of the behavior of algorithm $A$. By the linearity of expectation, we can sum over all pairs of reciprocal edges of $E'_T(A)$ to obtain 
\[
\E M'_T(A)\le\frac{2T}{n^2}M~.
\]
Finally, recalling that $M_T(A)\le M'_T(A)$, we can therefore conclude that the inequality
\[ 
\E M_T(A)\le\frac{2T}{n^2}M
\]
holds for any algorithm $A$, where the expectation is taken over the generation of function $\sigma$ for the input graph $(\langle B,G\rangle,E,\sigma)$.
\end{proof}




\subsection{Proof of Theorem~\ref{th:adversarialLB}}
\begin{proof}~\textit{[Sketch].}~
We first prove that $E^r_T(\om)$ is selected uniformly at random from $E^r$.
Thereafter, we will prove that $E_T(\om)$ contains in expectation $\Theta(T)$ pairs of reciprocal edges, i.e. $\E\, E^r_T(\om)=\Theta(T)$. Since these pairs are selected uniformly at random from $E^r$, 
this implies that $\E\, M_T(\om)=\Theta\left(\frac{T}{n^2}M\right)$. In fact, each match must necessarily be a pair of reciprocal edges, and will be selected this way with probability equal, up to a constant factor, to $\frac{T}{|E^r|}=\frac{T}{n^2}$.
In order to prove that $E^r_T(\om)=\Theta(T)$, 
we will define an {\em event related to each girl ${g} \in G$}. We will show that, throughout the algorithm execution, {\em each new} occurrence of this event is a sufficient condition to have a {\em new} pair of reciprocal edges in set $E^r_T(\om)$. 
We will find a lower bound for the expected number of times this event occurs, proving that it is equal to $\Theta(T)$, which implies that $\E E^r_T(\om)=\Theta(T)$. Since the pairs of reciprocal edges in $E^r_T(\om)$ are selected uniformly at random from $E^r$, this will allow us to conclude that the number of matches found is $\Theta\left(\frac{T}{|E^r|}M\right)=\Theta\left(\frac{T}{n^2}M\right)$. 

\bigskip

\om\ operates in steps $2_B$ and $2_G$ {\em without making any distinction} between any two boys or any two girls. In addition, the algorithm does not depend on the observed values of $\sigma$. Hence, \om\ can be seen as a random process dealing with sets $B$ and $G$ solely, where {\em each user is indistinguishable within the set s/he belongs to}. During any round $t$, the edge $(b,g)$ contained in $E_t(\om)\setminus E_{t-1}(\om)$ is selected uniformly at random from $B\times G$ at step $2_B$. At step $2_G$ of each round $t$, the algorithm selects uniformly at random either a boy from $B_{g,t}$ or from the whole set $B$. At each round $t$, $B_{g,t}$ is the result of the actions accomplished by \om\ during the previous rounds. As we pointed out, {\em all} these actions are carried out without making any distinction between any two users in $B$ and in $G$. Hence, during any given round $t$, if $B_{g,t} \neq \emptyset$, no boy is more likely to be part of $B_{g,t}$ than any other one. The probability that {\em any} pair $\{(b,g), (g,b)\}$
of reciprocal edge belongs to $E^r_T(\om)$, must therefore be the same for each pair of user $b \in B$ and $g \in G$ during any given round $t$.


\bigskip

Throughout this proof, for relevant event $\scE$, we denote by $t(\scE) \in [T]$ any round where event $\scE$ occurs. We also denote by $S(\scE) \subseteq [T]$ the set of all rounds where $\scE$ occurs.

We now define relevant events associated with each girl $g\in G$. 

\bigskip

\textbf{$\triangleright$ Definition of event $\scE_{g}(\Delta)$.}

Given any girl ${g} \in G$, and any round $t \le T-\Delta$ with $\Delta>0$, let $\scE_{g}(\Delta)$ be the conjunction of the following two events:
\begin{description} 
\item[\textbf{Event~}$\scEG_{g}(\Delta)$:] Girl ${g}$ is selected in Step ($1_G$) during both round $t$ and round $t+\Delta$, while she has never been selected in Step ($1_G$) during any round $t'$ such that  $t < t' < t+\Delta$;
\item[\textbf{Event~}$\scEB_{g}(\Delta)$:] {\bf (i)} There exists {\em one and only one} round $t'\in (t, t+\Delta]$ in which $g$ receives a feedback (uncovered during Step ($3_B$)), say feedback $\sigma(b'',g)$, and {\bf (ii)} we have $(b'',{g}) \not\in E_{t}(\om)$, i.e., this feedback was not uncovered until round $t$.\footnote
{
Recall that during the run of \om\ over $T$ rounds, for any given pair of users $(b,g)\in B\times G$, the feedback $\sigma(b,g)$ may be uncovered in Step ($3_B$) more than once.
}
\end{description}


We 
define the {\em occurrence round} $t(\scEG_{g}(\Delta))$ and $t(\scEB_{g}(\Delta))$ of event $\scEG_{g}(\Delta)$ and $\scEB_{g}(\Delta)$, respectively, 
as well as the occurrence round $t(\scE_{g}(\Delta))$ of the joint event $\scE_{g}(\Delta)$, as the round $t$ in the above definition of $\scEG_{g}(\Delta)$ and $\scEB_{g}(\Delta)$. To better clarify this definition, consider as an example the following sequence of triples 
\[
\langle {\mbox{Round\,, Feedback uncovered in Step ($3_B$)\,, Feedback uncovered in Step ($3_G$)}} \rangle 
\]
occurring from round $t$ to round $t+\Delta$, with $\Delta = 9$:
\begin{align*}
&\langle t,       ~~~~~~~~~~\sigma(b_{5}, g_{2})),          {\color{blue}~\sigma(g, b_{7})}\rangle \\ 
&\langle t+1,     ~~~\sigma(b_{4}, g_{3}),           ~~\sigma(g_{6}, b_{3}))\rangle \\
&\langle t+2,     ~~~\sigma(b_{8}, g_{9}),           ~~\sigma(g_{2}, b_{5}))\rangle \\
&\langle t+3,     ~~~{\color{blue}\sigma(b_{4}, g)}, ~~~~\sigma(g_{7}, b_{6}))\rangle \\
&\langle t+4,     ~~~\sigma(b_{7}, g_{9}),           ~~\sigma(g_{8}, b_{1}))\rangle \\
&\langle t+5,     ~~~\sigma(b_{4}, g_{1}),           ~~\sigma(g_{6}, b_{1}))\rangle \\
&\langle t+6,     ~~~\sigma(b_{3}, g_{3}),           ~~\sigma(g_{5}, b_{4}))\rangle \\
&\langle t+7,     ~~~\sigma(b_{4}, g_{3}),           ~~\sigma(g_{2}, b_{5}))\rangle \\
&\langle t+8,     ~~~\sigma(b_{3}, g_{5}),           ~~\sigma(g_{4}, b_{2}))\rangle \\
&\langle t+\Delta,~\sigma(b_{3}, g_{7}),           ~~~{\color{blue}\sigma(g, b_{6})})\rangle ~.
\end{align*}
If $\sigma(b_4,g)$ was never uncovered during any round $t'' \le t$, we say that events $\scEG_{g}(9)$, $\scEB_{g}(9)$ and $\scE_{g}(9)$ have occurred at round $t$, i.e., that $t(\scEG_{g}(9)) = t(\scEB_{g}(9)) = t(\scE_{g}(9)) = t$.
Observe that in this example girl $g$ is selected {\em twice} (round $t$ and round $t+\Delta$) and, during rounds $t+1, t+2, \ldots, t+\Delta$ she receives {\em one} feedback (uncovered in Step ($3_B$)), the one from boy $b_{4}$ at round $t+3$.

Finally, we define $\scE(\Delta)$ as the union of $\scE_{g}(\Delta)$ over all ${g} \in G$. 
\begin{fact}\label{f:indepDelta}
Events $\scEG_{g}(\Delta)$ and $\scEB_{g}(\Delta)$ are independent for all ${g} \in G$ and all $\Delta>0$, i.e. we always have  
$\Pr\,\scE_{g}(\Delta)=\Pr\,\scEG_{g}(\Delta)\cdot\Pr\,\scEB_{g}(\Delta)$.
\end{fact}
\begin{fact}\label{f:mutualDelta}
For any girl ${g} \in G$ and any pair positive integers $\Delta$ and $\Delta'$ with $\Delta\neq\Delta'$, we have that $\scE_{g}(\Delta)$ and $\scE_{g}(\Delta')$ are mutually exclusive. This mutual exclusion property also holds for events $\scEG_{g}(\Delta)$ and $\scEG_{g}(\Delta')$.
\end{fact}
\begin{fact}\label{f:mutualEv}
For any positive $\Delta$, given any pair of distinct girls $g$ and $g'$, we have that $\scE_{g}(\Delta)$ and $\scE_{g'}(\Delta)$ are mutually exclusive. This mutual exclusion property also holds for events $\scEG_{g}(\Delta)$ and $\scEG_{g'}(\Delta)$.
\end{fact}

Given any girl ${g}$, when $\scE_{g}(\Delta)$ occurs, we must have one of the two following {\em mutually exclusive} consequences $\scC_1$ and $\scC_2$, namely, any occurrence of $\scE_{g}(\Delta)$ implies either $\scC_1$ or $\scC_2$ but not both: 
When we disclose the preference of boy $b''$ for girl ${g}$ during round $t' \in (t(\scE_{g}(\Delta)),t(\scE_{g}(\Delta))+\Delta]$ we have either $({g},{b''})   \not\in E_{t'-1}(\om)$ or $({g},{b''})   \in E_{t'-1}(\om)$. This in turn implies:

\begin{description} 


\item[Consequence $\scC_1$\,: ] $({g},{b''}) \not\in E_{t'-1}(\om)$. \newline  
Boy $b''$ must belong to $B_{g,\widetilde{t}}$ (Step ($2_G$)) for all rounds $\widetilde{t}\in [t', t+\Delta]$. Since $B_{g,t+\Delta}$ is not empty, because it contains at least boy $b''$,  a {\em new} pair $\{(\widetilde{b},{g}), ({g},\widetilde{b})\}$ of reciprocal edges is uncovered (note that we need not have $\widetilde{b}\equiv b''$, since $B_{g,t+\Delta}$ may also include some other boys besides $b''$). Hence, set $E^r_{t+\Delta}(\om)\setminus E^r_{t+\Delta-1}(\om)$ must contain $\{(\widetilde{b},{g})  , ({g},\widetilde{b})\}$.



\item[Consequence $\scC_2$\,: ] $({g},{b''})   \in E_{t'-1}(\om)$. \newline In this case \om\ finds the {\em new} pair $\{(b'',{g})  , ({g},{b''})  \}$ of reciprocal edges during round $t'$ (Step ($2_B$)), i.e., the set $E^r_{t'}(\om)\setminus E^r_{t'-1}(\om)$ must contain $\{(b'',{g})  , ({g},{b''})  \}$. 
\end{description}


Thus, taking into account that $\scE_{g}(\Delta)$ is a sufficient condition for $\scC_1 \lor \scC_2$, we can always associate a new occurrence of $\scE_{g}(\Delta)$ with a {\em distinct} pair $\{(b,{g})  , ({g},b)\}$ of reciprocal edges in $E^r_T(\om)$.
Hence, \om\ finds at least $|S(\scE(\Delta))|$ {\em distinct} pairs of reciprocal edges, i.e., $E^r_T(\om) \ge |S(\scE(\Delta))|$.

\bigskip

Let now $\alpha\in(0,1)$ be a constant parameter. We focus on computing 
\[
\E \sum_{\Delta=\alpha n}^{n}|S(\scE(\Delta))|~.
\]
We set for brevity
\[
\scE(\alpha n, n) = \cup_{\Delta \in [\alpha n, n]}\scE(\Delta)~
\]
and, given any girl ${g}$,
\[
\scE_{g}(\alpha n, n) = \cup_{\Delta \in [\alpha n, n]}\scE_{g}(\Delta)~, \quad \scEB_{g}(\alpha n, n) = \cup_{\Delta \in [\alpha n, n]}\scEB_{g}(\Delta), \quad 
\scEG_{g}(\alpha n, n) = \cup_{\Delta \in [\alpha n, n]}\scEG_{g}(\Delta)~.
\]
We recall we defined the occurrence round $t(\scE_{g}(\Delta))$ as the first of the $(\Delta+1)$-many rounds related to definition of event $\scE_{g}(\Delta)$. We define the occurrence rounds $t(\scE_{g}(\alpha n, n))$ and $t(\scE(\alpha n, n))$ in a similar manner, as the {\em earliest round} $t$ when, respectively, $\scE_{g}(\Delta)$ and $\scE(\Delta)$ occurs over all $\Delta \in [\alpha n, n]$.
\begin{fact}\label{f:independenceAlphaN}
Given any $\alpha \in (0,1)$ and any ${g} \in G$, Fact~\ref{f:indepDelta} and Fact~\ref{f:mutualDelta} ensure that
events $\scEG_{g}(\alpha n, n)$ and $\scEB_{g}(\alpha n, n)$ are independent, i.e., we always have  
$\Pr\,\scE_{g}(\alpha n, n)=\Pr\,\scEG_{g}(\alpha n, n)\cdot\Pr\,\scEB_{g}(\alpha n, n)$.
\end{fact}
\begin{fact}\label{f:EvMutualEqual}
Given any $\alpha \in (0,1)$, and any pair of distinct girls ${g'}$ and $g''$, Fact~\ref{f:mutualEv} and Fact~\ref{f:mutualDelta} ensure that $\scE_{g'}(\alpha n, n)$ and $\scE_{g''}(\alpha n, n)$ are mutually exclusive. Furthermore, Fact~\ref{f:independenceAlphaN}, together with the definition of $\scEG_{g}(\alpha n, n)$ and $\scEB_{g}(\alpha n, n)$ for any girl $g\in G$, ensures $\Pr\,\scE_{g'}(\alpha n, n)=\Pr\,\scE_{g''}(\alpha n, n)$.
\end{fact}



We now prove that any constant $\alpha \in (0,1)$ leads to $\E |S(\scE(\alpha n, n))|=\Theta(T)$. This implies $M_T(\om)=\Theta\left(\frac{T}{n^2}M\right)$, since $E^r_T(\om)$ is made up of pairs of reciprocal edges which are selected uniformly at random from $E^r$.



In order to estimate $\E |S(\scE(\alpha n, n))|$, we will lower bound the probability $\Pr\,\scE(\alpha n, n)$, which in turn will require us to lower bound $\Pr\,\scE_{g}(\alpha n, n)$.  
Since in Step ($1_G$) a girl is selected uniformly at random from $G$, for any $g \in G$ we can write :


%
\begin{align}
\forall g \in G ~~~ \Pr\,\scEG_{g}(\alpha n, n)
&=\sum_{\Delta=\alpha n+1}^{n-1} \Pr\,\scEG_{g}(\Delta)\label{eq:sumscEGv}\\
&=\sum_{\Delta=\alpha n+1}^{n-1} 
\frac{1}{n^2} \left(1-\frac{1}{n}\right)^{\Delta-1}\nonumber\\
&\ge\frac{1}{n^2}\sum_{\Delta=\alpha n+1}^{n-1}
\left(1-\frac{1}{n}\right)^{\Delta}\nonumber\\
&=\frac{1}{n^2}
\left(
\frac{1-(1-n^{-1})^{n}}{1-(1-n^{-1})}-
\frac{1-(1-n^{-1})^{\alpha n+1}}{1-(1-n^{-1})}
\right)\nonumber\\
&=\frac{1}{n^2}
\left(
\frac{(1-n^{-1})^{\alpha n+1}-(1-n^{-1})^{n}}{n^{-1}}
\right)\nonumber\\
&\sim_{n \to \infty}
\frac{
e^{-\alpha}-e^{-1}
}{n}\label{eq:finalEGv}~,
\end{align}
where in Equation~(\ref{eq:sumscEGv}) we used Fact~\ref{f:mutualDelta}.

We now bound $\Pr\,\scEB_{g}(\alpha n, n)$ for all ${g} \in G$. 
%
%
We define the event $\scEB_{g, b''}(\Delta)$ based on the definition of $\scEB_{g}(\Delta)$  provided at the beginning of the proof.
Given any boy $b''\in B$, event $\scEB_{g, b''}(\Delta)$ occurs whenever: 
{\bf (i)} there exists {\em one and only one} round $t'\in (t, t+\Delta]$ in which 
$g$ receives a feedback (uncovered in Step ($3_B$)) from $b''$, and 
{\bf (ii)} $g$ does not receive any feedback from any other boy during any round in $(t, t+\Delta]$, and 
{\bf (iii)} we have $(b'',{g}) \not\in E_{t}(\om)$, i.e., this feedback was not uncovered until round $t$.

Observe that, by this definition, we have 
$\scEB_{g}(\Delta)\equiv\cup_{b''\in B} \scEB_{g, b''}(\Delta)$ --- see the definition of $\scEB_{g}(\Delta)$ provided above to compare events $\scEB_{g}(\Delta)$ and $\scEB_{g, b''}(\Delta)$.
Now, given any girl $g\in G$, we define $\scEB_{g,{b''}}(\alpha n, n)\defeq\cup_{\Delta \in [\alpha n, n]}\scEB_{g,{b''}}(\Delta)$.



\begin{fact}\label{f:mutualEqualEBvj}
Given any girl ${g} \in G$, for each pair of distinct boys $b',b'' \in B$, events $\scEB_{{g},{b'}}(\alpha n, n)$ and $\scEB_{{g},b''}(\alpha n, n)$ are  mutually exclusive by their definition. Furthermore, Step ($1_B$) and Step ($2_B$) ensure that $\Pr\,\scEB_{{g},{b'}}(\alpha n, n)=\Pr\,\scEB_{{g},b''}(\alpha n, n)$. Mutual exclusion also holds for events $\scEB_{{g'},{b}}(\alpha n, n)$ and $\scEB_{g'',{b}}(\alpha n, n)$ for any $b \in B$ and pair of distinct ${g'}, g'' \in G$.
\end{fact}

We can now conclude that, for any given occurrence round of $\scEB_{g}(\alpha n, n)$ and any integer $T\in \{n, n+1, \ldots,  n^2\}$, we have:



\begin{align}
\forall g \in G~~~\Pr\,\scEB_{g}(\alpha n, n)
&=
\Pr\,\left(\cup_{b'' \in B}~\scEB_{{g},{b''}}(\alpha n, n)\right)\label{eq:EBinitalLB}\\
&=n~
\Pr\,\scEB_{{g},{b''}}(\alpha n, n)\label{eq:mutualEqualEBvj}\\
&\ge n
\left(1-\frac{1}{n^2}\right)^{T}
\frac{\min(\alpha n, n)}{n^2} \left(1-\frac{1}{n}\right)^{\max(\alpha n, n)-1}\nonumber\\
&\ge n
\left(1-\frac{1}{n^2}\right)^{n^2}
\frac{\alpha}{n} \left(1-\frac{1}{n}\right)^{n}\nonumber\\
&\sim_{n \to\infty} \alpha e^{-2}\label{eq:EBfinalLB}~,
\end{align}
where in Equation~(\ref{eq:mutualEqualEBvj}) we used Fact~\ref{f:mutualEqualEBvj}.

We can finally bound the probability of event $\scE(\alpha n, n)$ (as $n$ grows large):
\begin{align}
\Pr\,\scE(\alpha n, n)
&=\Pr\,\scEG(\alpha n, n)\cdot\Pr\,\scEB(\alpha n, n)\label{eq:indAlphaN}\\
&=\left(\Pr\,\cup_{{g} \in G}\scEG_{g}(\alpha n, n)\right)\cdot
\alpha e^{-2}\label{eq:EBLBused}\\
&\ge\left(e^{-\alpha}-e^{-1}\right)
\cdot\alpha e^{-2}\label{eq:finalAlpha}~,
\end{align}
where in Equation~(\ref{eq:indAlphaN}) we used Fact~\ref{f:independenceAlphaN}, in Equation~(\ref{eq:EBLBused}) we used
the chain of inequalities~(\ref{eq:EBinitalLB})---(\ref{eq:EBfinalLB}), and in Equation~(\ref{eq:finalAlpha}) we used Fact~\ref{f:EvMutualEqual}, together with the chain of inequalities ~(\ref{eq:sumscEGv})---(\ref{eq:finalEGv}).


Let us denote for brevity $\alpha e^{-2}\left(e^{-\alpha}-e^{-1}\right)$ by $c(\alpha)$. We clearly have $c(\alpha)>0~\forall \alpha \in (0,1)$. Event $\scE(\alpha n, n)$ can occur at any round $t \le T-n$. Recall that we denoted by $S(\scE)$ the set of rounds where event $\scE$ occurs. 

For all integers $T$ such that $T-n=\Omega(n)$ we now have:
\begin{align}\label{soamLB}
\E M_T(\om)
&=\frac{\E|E^r_T(\om)|}{|E^r|}M\nonumber\\
&\ge\frac{\E|S(\scE(\alpha n, n))|}{n^2}M\\ 
&\ge\frac{(T-n)~\Pr\,\scE(\alpha n, n)}{n^2}M\label{eq:linExpUsed}\\ 
&\ge\frac{(T-n)~c(\alpha)}{n^2}M\nonumber\\ 
&=\Theta\left(\frac{T}{n^2}M\right)~,
\end{align}
where in Equation~(\ref{eq:linExpUsed}) we used the linearity of expectation of events $\scE(\alpha n, n)$, by summing $\Pr\,\scE(\alpha n, n)$ over the first $T-n$ rounds.
\end{proof}

\subsection{Proof of Theorem~\ref{th:smileMatchMequalMstar}}
\begin{proof}~
Let $T_I$ and $T_{II}$ be the number of rounds used during Phase I and Phase II, respectively. 
Thus we have $T_{II}=T-T_I$. The proof structure is as follows. After bounding $C^G$ and $C^B$, we will show that $T_I=\scO\left(n(C^G+C^B+S')\right)$. 
Note that this implies $T_I=o(T)$ for any $T$ satisfying the lower bound $T=\omega\left(n(C^G+C^B+S')\right)$. Then we will prove that, during phase II, $T_{II}$-many rounds are sufficient to serve in Step (1) each user a total number of times  which is w.h.p. larger than $\max_{u\in B\cup G}\degree_{\scM'}(u)$,
where $\scM'$ is the matching graph estimated by \smi. This fact can be proven by combining the two conditions $M^*_T=M$ (which is assumed to hold with high probability), and $M =\omega\left(\frac{n^2\log(n)}{S}\right)$. Hence, after $o(T)$-many rounds of Phase I, \smi\ can start to {\em greedily} simulate the Omniscent Matchmaker on the estimated $\scM'$ during Phase I. Finally, we prove that the number of edges of $\scM$ which are also contained in $\scM'$
is $\Theta(M)$, which implies that during phase II \smi\ will uncover w.h.p. $\Theta(M)$ matches. This will conclude the proof.

Now, for the sake of this proof, 
we will focus on set $B$ and Steps $(1_B)$, $(2_B)$ and $(3_B)$. The corresponding claims for $G$ and Steps $(1_G)$, $(2_G)$ and $(3_G)$ are completely symmetrical.

We start by briefly recalling the parts of the algorithm which are relevant for this proof.
We define the boys and girls as arranged in sequences $\langle b_1, b_2, \dots b_n \rangle$ and
$\langle g_1, g_2, \dots g_n \rangle$. 
Let $G^r_t$ be the set of the cluster representive girls found by \smi\ during all rounds up to $t$. Let $t(g)$ be the round in which the girl $g$ is included in a subset of $G^r_T$ during the execution of the algorithm, i.e., the round when she becomes a cluster representive girl. 
The construction of $G^r_T$ is accomplished in a {\em greedy} fashion. Specifically, if in round $t$ of Phase I {\em all} girls $g_1, g_2, \ldots, g_i$ are either part of $G^r_t$ or are included in a cluster then, at the beginning of round $t+1$, \smi\ picks the next girl $g_{i+1}$. 
Note that $g_{i+1}$ can be any member of $G$ who has not been processed yet. Thereafter, \smi\ estimates whether the feedback received by $g_{i+1}$ is similar to the one of at least one cluster representative girl found so far. More precisely, after having collected $S'$ feedbacks for her, \smi\ uses a randomized strategy relying on Lemma~\ref{le:hammingDistance}. 
Let then $t'$ be the round in which $|F_{g_{i+1}}|$ becomes equal to $S'$. (Recall that, for each user $u\in B\cup G$, $F_u$ is the set of all feedbacks received until the current round.) 
If at round $t'$ we have that for all $b\in F_{g_{i+1}}$ there exists a girl $g^r\in G^r_{t'-1}$ 
such that $\sigma(b,g_{i+1})=\sigma(b,g^r)$, then $g_{i+1}$ is included in the same cluster of $g^r$. Otherwise, \smi\ collects feedback for $g_{i+1}$ until we have $|F_{g_{i+1}}|=\frac{n}{2}$, and then $g_{i+1}$ becomes a new cluster representative girl.  

In order to prove that $T_I=\scO\left(n(C^G+C^B+S')\right)$, we need to upper bound $|C^G|=|C^G_T|$ and $|C^B|=|C^B_T|$.
As in Section~\ref{s:successfulApproach}, we denote by $\bB$ the matrix of all ground truth preferences of the boys. Namely, 
for each $i,j\in [n]$, $\bB_{i,j}$ is equal to $\frac{1}{2}(1+\sigma(b_i, g_j))$.
Given girl $g_j$, we denote by $\bg_j$ the vector of feedback received by $g_j$, i.e. the $j$-th column vector of $\bB$.  
Let $C_{\rho}^B$ be the covering number of radius $\rho$ of all the column vectors of $\bB$. Given two $0\,$-$1$
vectors $\bv$ and $\bv'$, we denote by  $d(\bv,\bv')$ the Hamming distance between them. Given any non-negative integer 
$\rho$, let $\scB_{\rho}(g)$ be the set of 
$\bv$ such that $d(\bg,\bv)\le\rho$, i.e. the {\em ball} centered at $\bg$.
Finally, let $G^{r,\rho}_T\subseteq G^r_T$ be the set of all girls $g$ included by \smi\ in $G^r_T$ while there exists at least one girl $g^r\in G^r_{t(g)-1}$ such that $\bg$ belongs to ball $\scB_{\rho}(g^r)$ centered at $\bg^r$.


In this proof, we single out subset $G^{r,\rho}_T\subset G^r_T$ since, in order to upper bound $|G^r_T|$, it is convenient to bound $|G^{r,\rho}_T|$ and 
$|G^r_T\setminus G^{r,\rho}_T|$ separately, and then use the sum of these two bounds to limit $|G^r_T|$.
Notice that by its very definition, $G^{r,\rho}_T$ can be seen as containing all girl representative members $g$ of $G^r_T$ satisfying the following property: Given {\em any} radius $\rho$, there exists at least one girl $g^r\in G^r_{t(g)-1}$ such that $\bg$ belongs to the ball $\scB_{\rho}(g^r)$ centered at $\bg^r$. This property states that, given any radius $\rho$, \smi\ creates a {\em new} representative girl $g$ instead of including $g$ into the cluster of $\bg^r$. In fact, after round $t(g)$, both $\bg\in\scB_{\rho}(g^r)$ and $\bg^r\in\scB_{\rho}(g)$ will simultaneously hold because $d(\bg,\bg^r)\le\rho$. 
This event may happen because while \smi\ is looking for a cluster including $g$, there exists at least one boy $b''\in B(g,S)\cap F_{g^r}$ (see Section~\ref{ss:smile} -- Phase I) such that $\sigma(g,b'')\neq\sigma(g^r,b'')$. Clearly, the larger the considered $\rho$, the more frequent this event is. Since \smi\ operates without considering any specific radius $\rho$, this fact holds for {\em all} values of $\rho$.



Taking into account the greedy way \smi\ constructs $G^r_T$, we have $|G^r_T\setminus G^{r,\rho}_T|\le C^G_{\rho/2}$. 
In fact, given any optimal\footnote
{
By \enquote{optimal} we mean here a covering having a number of balls exactly equal to the covering number.
}
$\frac{\rho}{2}$-covering $\scC^B_{\rho/2}$,
by the definition of $G^{r,\rho}_T$, we know that
{\em at most one} girl of $G^r_T\setminus G^{r,\rho}_T$ can be included in any ball of $\scC^B_{\rho/2}$.
Now, since we know that $|G^r_T\setminus G^{r,\rho}_T|\le C^G_{\rho/2}$, in order to upper bound $|G^r_T|$ in terms of $C^G_{\rho/2}$, we can bound $|G^{r,\rho}_T|$. 
A union bound shows that the probability that any girl $g$ belongs to $G^{r,\rho}_T$ is upper bounded by $\frac{\rho S'}{n}$. In fact, from the definition of $G^{r,\rho}_T$, we know that there is already at least one girl $g^r$ in $G^r_{t(g)-1}$ such that $\bg\in\scB_{\rho}(g^r)$. 

Let $F_{S',g}$ be the set of feedbacks received by $g$ when $|F_g|$ becomes equal to $S'$ and \smi\ verifies whether $g$ can be part of a previously discovered cluster. For each boy $b\in F_{S',g}$, the probability that $\sigma(b,g)\neq\sigma(b,g^r)$ is at most $\frac{\rho}{n}$. The probability that $\sigma(b,g)\neq\sigma(b,g^r)$ holds for {\em all} $b\in F_{S',g}$ can therefore be bounded from above by 
$|F_{S',g}|\frac{\rho}{n} = \frac{\rho S'}{n}$. 
Since $|G|=n$, the cardinality of $G^{r,\rho}_T$ is therefore upper bounded by $\rho S'$ in expectation. Applying now a Chernoff bound, and taking into account that $S'> S \ge \log n$ and that the radius $\rho$ is at least $1$ when it is not null, we obtain that the upper bound 
\[
|G^{r,\rho}_T|\le \rho S' + 2\sqrt{S'\rho\log n}
\] 
holds w.h.p. 
Hence, we conclude that 
\[
C^G=|G^r_T|\le C^G_{\rho/2} + \rho S'+2\sqrt{\rho S'\log n}\le C^G_{\rho/2} + 3\rho S'~,
\]
holds w.h.p. {\em for all} non-negative values of the radius $\rho$. Since $C^G$ is clearly upper bounded by $n$, we can finally write
\[
C^G\le\min\left\{\min_{\rho\ge 0}\left(C^G_{\rho/2} + 3\rho S'\right),n\right\}~.
\]
By symmetry, we can use the same arguments as above for bounding $C^B$. This concludes the first part of the proof.

We now prove that $T_I=\scO\left(n(C^G+C^B+S')\right)$.
Let $T_I^B$ be the number of rounds during which \smi\ asks for feedback to boys in Phase I. $T_I^B$ is bounded by the sum of the number of rounds used to obtain $\frac{n}{2}$ feedbacks for each girl in $G^r_T$, and the number of rounds to obtain $S'$ feedbacks for each girl in $G\setminus G^r_T$. These two quantities are upper bounded w.h.p. by $\scO(n |G^r_T|)$ and $\scO(S' |G\setminus G^r_T|)=\scO(S' n)$, respectively.
Hence, the total number of rounds \smi\ takes for asking all feedbacks from boys during Phase I is upper bounded w.h.p. by 
\[
\scO\left(n(C^G+S')\right)~.
\]
Since $T_I\le T^G_I+T^B_I$ and $T^B_I=\scO\left(n C^B+S n\right)$, we conclude that 
\begin{equation}\label{e:boundT1}
T_I=\scO\left(n(C^G+C^B+S')\right)~.
\end{equation}
We now show that under the assumptions of the theorem, the strategy of Phase II yields w.h.p. to match $\Theta(M)$ users.
For each cluster representative member, the number of feedbacks obtained by selecting uniformly at random users from the other side during Step $(1)$ is equal to $\frac{n}{2}$. Hence, if we disregarded the number of mispredicted matches, \smi\ would recover w.h.p. at least $\frac{1}{4}M-\scO\left(\sqrt{M\log n}\right)$ matches selected uniformly at random from $E_{\scM}$. The number of mispredicted matches quantified by Lemma~\ref{le:hammingDistance} is equal to 
$\scO\left(\frac{n\log(n)}{S}\right)$ per user, which are caused by the fact that $\scM$ is not recovered exactly by \smi, but only in an approximate manner.
Denote by $\scM$'s the approximation to $\scM$s computed by Phase II. Using a Chernoff bound and the conditions
$M=\omega\left(\frac{n^2\log(n)}{S}\right)$ and $S<n$ (which together imply $M=\omega\left(n\log(n)\right)$ as $n\to\infty$),
we have that the total number $|E_{\scM}\triangle E_{\scM'}|$ of mispredicted matches satisfies w.h.p.
\[
|E_{\scM}\triangle E_{\scM'}|\le\frac{3}{4}M+\scO\left(\sqrt{M\log n}\right)+\scO\left(\frac{n^2\log(n)}{S}\right)=\frac{3}{4}M+o\left(M\right)~,
\]
where $E_{\scM}\triangle E_{\scM'}$ is the symmetric difference between the edge sets of $\scM$ and $\scM'$.

Set for brevity  $d_{\max} = \max_{u \in V}\degree_{\scM}(u)$. 
We now claim that
\begin{equation}\label{e:claim}
d_{\max}-o(d_{\max})\ge\degree_{\scM'}(u)
\end{equation}
holds w.h.p. for each user $u\in B\cup G$.
The operations performed by \smi\ guarantee that w.h.p. $\degree_{\scM'}(u)-\degree_{\scM}(u)=\scO\left(\frac{n\log n}{S}\right)$ holds for all $u\in B\cup G$. In fact, for each user $u\in B\cup G$, the total number of users $u'$ on the other side who dislike $u$ and are adjacent to $u$ in $\scM'$, is upper bounded w.h.p. by $\scO\left(\frac{n\log n}{S}\right)$, as Lemma~\ref{le:hammingDistance} guarantees. 
%
%
Now, we have w.h.p.
\begin{align*}
\degree_{\scM'}(u) 
&\le \frac{1}{4}d_{\max}+\scO\left(\sqrt{d_{\max}\log(n)}\right)+\scO\left(\frac{n\log n}{S}\right)\\
&=\frac{1}{4}d_{\max}+o(d_{\max})+o\left(\frac{M}{n}\right)\\
&=\frac{1}{4}d_{\max}+o(d_{\max})\\
&\le d_{\max}-o(d_{\max}),
\end{align*}
where the term $\scO\left(\sqrt{d_{\max}\log(n)}\right)$ arises from the application of a Chernoff bound and we took into account that 
$M=\omega\left(\frac{n^2\log(n)}{S}\right)$, combined with $S<n$, implies $d_{\max}=\omega\left(\log(n)\right)$. This concludes the proof of (\ref{e:claim}).




During phase II, \smi\ matches pairs of users corresponding to $E_{\scM'}$ in a greedy way. If we can show that each user $u$ is served w.h.p. at least $\degree_{\scM'}(u)$ times then we are done.
Now, since $M^*_T=M$ w.h.p., the Omniscient Matchmaker must be able to match w.h.p. all the users corresponding to $E_{\scM}$ in $T$ rounds. This implies that in Steps $(2_B)$ and $(2_G)$ each user $u \in B\cup G$ is served w.h.p. at least $d_{\max}$ times during the $T$ rounds. Hence, 
each user is served w.h.p. at least $d_{\max}-o(d_{\max})$ times during the last $T-T_I=(1-o(1))T$ rounds, where $T_I$ is the time used by phase I. 
Recalling now (\ref{e:claim}) and (\ref{e:boundT1}),
we conclude that $T=\omega\left(n(C^G+C^B+S')\right)$ rounds are always sufficient to serve each user $u$ at least $\degree_{\scM'}(u)$ times, thereby completing the proof.
%
\end{proof}

\subsection{Proof of Theorem~\ref{th:smileTimeNeeded}}
\begin{proof}~[\textit{Sketch.}]~
Term $M$ in the lower bound clearly derives from the fact that we need to match $\Theta(M)$ users. When $M$ is the dominant term, the bound is therefore trivially true. In the sequel, we thus focus on the case $M=o\left(n(C^G_0+C^B_0)\right)$, i.e., when the dominant term is $n(C^G_0+C^B_0)$.

We show how to build a sign function $\sigma$ such that the number of rounds needed to uncover $\Theta(M)$ matches is $\Omega\left(n(C^G_0+C^B_0)\right)$.
First of all, we set $\sigma(g,b)=1$ for all $g \in G$ and all $b \in B$. This implies $C^B_0=1$. The matches depend therefore solely on the boy preference matrix $\bB$.  We create an instance of $\bB$ where, for $\rho=0$, the number of girls belonging to each cluster of the columns of $\bB$ is equal to $\frac{n}{C^G_0}$, i.e., all these clusters of girls have the same size $\frac{n}{C^G_0}$. Let $d$ be any divisor of $n$.
Without loss of generality, consider $\bB$ after having rearranged its columns in such a way that all column indices are grouped according to the girl clustering. More precisely, given any $i\in \{0,1,\ldots, d-1\}$, the column indices of $\bB$ in the range $\left[1+i\frac{n}{d}, (i+1)\frac{n}{d}\right]$ belong to the same girl cluster. We obtain this way a block matrix $\bB$ made up of $(n d)$-many blocks, where each block is a submatrix having $1$ row and $\frac{n}{d}$ columns. 
We then choose uniformly at random 
$\left\lfloor\frac{m d}{n}\right\rfloor$ blocks, and set equal to $1$ all entries in each selected block. Finally, we set all the remaining entries of $\bB$ to $0$. With this random assignment, we have that in expectation $C^G_0$ equals $d$. In fact, since $m\in\left(n\log(n),n^2-n\log(n)\right)$, we can always select at least $d\log(d)$-many blocks. By using a classical Coupon Collector argument, we see that in expectation we have at least one block of entries equal to $1$ (and one block of entries equal to $0$, both selected uniformly at random) per set of $\frac{n}{d}$ columns grouped together as explained above.  Note also that this way we have $m-\frac{n}{C^G_0}< M\le m$, which is equivalent to $m-\frac{n}{C^G_0+C^B_0-1}< M\le m$, since $C^B_0=1$. 

Assume now $T=o\left(n(C^G_0+C^B_0)\right)$, which is equal to $o\left(n(C^G_0)\right)$ in our specific construction. In this case, for any matchmaking algorithm $A$, the number of feedbacks from boys revealed in Steps ($3_G$) and ($3_B$) must be $o\left(n(C^G_0+C^B_0)\right)=o\left(n(C^G_0)\right)$. This implies that, in expectation, the fraction of matches that are {\em not} covered by $A$ is asymptotically equal to $1$ as $n\to\infty$. Hence, our construction of $\sigma$ shows that in order to uncover $\Theta(M)$ matches in expectation, it is necessary to have $T=\Omega\left(n(C^G_0+C^B_0)+M\right)$, as claimed.
\end{proof}

\subsection{Proof of Theorem~\ref{th:smileCompComplexity}}
\begin{proof}~[\textit{Sketch.}]~
We describe an efficient implementation of \smi\, analyzing step by step the time and space complexity of the algorithm. 
Without loss of generality, we focus on $B$ and the operations performed on matrix $\bB$. Similar operations can be performed on $G$ and $\bG$, so that
the total time and space complexity of the algorithm will be obtained by simply doubling the complexities computed within this proof (this will not affect the final results because of the big-Oh notation).

We create a balanced tree $\scT$ whose records contain the feedbacks collected for all cluster representative members of $G^r$ during Phase I. More precisely, $\scT$ 
contains all ordered sets of indices of $\bB$'s columns according to their lexicographic order. We insert each column one by one reading all its binary digits.
This way, we can quickly insert new elements while maintaining them sorted even within each node of $\scT$. 
At the end of this process, we will have $C^G$ records. 
The resulting time complexity is $\scO(n\,C^G \log C^G)$; the space complexity is $\scO(n\,C^G)$.


Each time we collect $S'$ feedbacks for a girl $g$, we check whether we can put her in a cluster based on the available information. We look for one girl $g^r \in G^r$ such that we have $\sigma(b,g^r)=\sigma(b,g)$ for all $b \in F_{g^r}\cap F_g$. If do not find any such girl, we continue to collect feedback for $g$ until $|F_g|=\frac{n}{2}$, and thereafter we insert $g$ in $G^r$. This operation is repeated for all girls except for the first one. This is the computational bottleneck of the whole implementation. Overall, it 
takes $\scO(n) \cdot \scO(S') \cdot \scO(C^G)=\scO(n\,C^G\,S)$ time.
The space complexity is still $\scO(n\,C^G)$, because of the use of tree $\scT$.

At the end of this phase, we create a matrix $\wt{\bB} \in \{0,1\}^{n\times C^G}$ containing all the columns in $\scT$ in the same order. We also create two
other ancillary data-structures: \textbf{(i)}  An $n$-dimensional array $A_B$ where each record contains an integer in $\{1,\ldots,C^B\}$, representing the estimated cluster of each boy. Array $A_B$ allows us to get in constant time the estimated cluster of each boy. \textbf{(ii)} A $C^G$-dimensional array $A'_B$, where each record represents a distinct cluster of girls. 
The $j$-th entry $A'_B[j]$ of $A'_B$ contains the ordered list of the indices of all girls belonging to the $j$-the estimated cluster.

Symmetrically, for the girl preference matrix $\bG$, we will have matrix $\wt{\bG}$ and arrays $A_G$ and $A'_G$.
Finally, we create a $C^B$-by-$C^G$ matrix $\bM$ which can be exploited in Phase II to match users according to the information collected during Phase I. Matrix $\bM$ represents, in a very compact form, the approximation to the matching graph $\scM$ computed by Phase I. 
Specifically, entry $\bM_{i,j}$ contains two ordered lists of user indices, $L_B(i,j)$ and $L_G(i,j)$. The integers in $L_B(i,j)$ correspond to {\em all} boy indices that belong to the $i$-th cluster of $B$ and that, according to what the algorithm estimates, are matching girls in the $j$-th cluster of girls. Symmetrically, $L_G(i,j)$ contains {\em all} the indices of the girls belonging to the $j$-th cluster of girls matching boys of the $i$-th cluster. It is not difficult to see that using the data-structures described so far, this matching matrix $\bM$ can be generated by reading all elements of $\wt{\bB}$ and $\wt{\bG}$ only once, and its construction thus requires only $\scO\left(n\,(C^G+C^B)\right)$ time. The space complexity of the matching matrix $\bM$ is again $\scO\left(n\,(C^G+C^B)\right)$.
To see why, first observe that the number of entries of $\bM$ is $C^G\cdot C^B< n(C^G+C^B)$. As for the space needed to store the boy and girl lists contained in the entries of $\bM$, consider the following. Let us focus on boys only, a similar argument can be made for girls. List $L_B(i,j)$, stored in $\bB_{i,j}$, must be a 
subset of the $i$-th estimated cluster of $B$. Since $B$ is partitioned by \smi\ into $C^B$-many estimated clusters, call these clusters $B_1, \ldots, B_{C^B}$, we have that 
the total number of items contained in all the lists of the $i$-th row of $\bM$ can be upper bounded by $|B_i|\cdot C^G$. Thus, the total number of items contained in the lists of boys in $\bM$ can in turn be upper bounded by 
\[
\sum_{1\le i \le C^B} |B_i|\cdot C^G = |B|\cdot C^G=n\cdot C^G~.
\] 
Hence, the space needed to store $\bM$ is bounded by 
\[
\scO\left(C^G\,C^B+n\,C^G+n\,C^B\right)=\scO\left(n\,(C^G+C^B)\right)~,
\] 
as claimed.

During Phase II, we match users according to the information obtained from Phase I. 
%
The procedure is greedy, and can be efficiently implemented by maintaining, for each $b \in B$, a pointer $p_b$ that can only move forward to the corresponding  
row of $\bM$. More precisely, $p_b$ scans the estimated matches for $b$ contained in the corresponding row of $\bM$. Without loss of generality, assume $b$ is contained in the $i$-th estimated cluster of boys, and that $L_B(i,j)$ contains $b$.
During each round where boy $b$ is selected (in some Step ($1_B$)), pointer $p_b$ moves forward in the list $L_G(i,j)$, where $\bM_{i,j}$ is the current entry processed by \smi\ during Phase II for $b$. If during the last round where $b$ was selected, $p_b$ was pointing to the last element of $L_G(i,j)$, then we continue to increment $j$ until we find an entry $\bM_{i,j'}$ such that the associated list of boys $L_B(i,j')$ contains $b$. In order to find such entry $j'$, we perform $(j'-j)$-many binary searches over the $j'-j$ lists $L_B(i,j+1), L_B(i,j+2), \ldots, L_B(i,j')$. Thereafter, we make $p_b$ point to the first girl in list $L_G(i,j')$.
%
When $p_b$ reaches the end of the list of girls $L_G(i,C^G)$ of the last column of $\bM$, \smi\ predicts arbitrarily in all subsequent rounds where $b$ is selected.


The total running time for Phase II 
is $\scO\left(T+n\,(C^G+C^B)\,\log n\right)$, where term $\scO\left((C^G+C^B)\,\log n\right)$ is due to the dichotomic searches performed in the lists of $\bM$ for each user of $B\cup G$.
To see why, let us refer to a specific boy $b$: The number of operations performed during Phase II is either constant, when $p_b$ moves forward inside a list of girls of $\bM$, or $\scO\left((j'-j)\log(n)\right)$, when \smi\ is looking for the next list of boys $L_B(i,j')$ containing $b$ 
starting from the lists of entry $\bM_{i,j}$. Hence the overall time complexity becomes
\[
\scO(T+n\,(C^G+C^B)\,\log n+n\,S\,(C^G+C^B))=\scO(T+n\,(C^G+C^B)\,S)
\] 
where we used $S\ge \log(n)$. 

As for the amortized time per round, when $T=\omega\left(n(C^G+C^B)+\frac{n^3\log(n)}{M}\right)$ this can be calculated as follows. 
Since $S=\Theta\left(\frac{n^2\log(n)}{M}\right)$, the overall time complexity becomes 
$\scO\left(T+(C^G+C^B)\frac{n^3\log(n)}{M}\right)$. Dividing by $T=\omega\left(n(C^G+C^B)+\frac{n^3\log(n)}{M}\right)$, 
we immediately obtain
\begin{align*}
&\scO\left(\frac{T+(C^G+C^B)n^3\log(n)/M}{T}\right)\\
&=\scO\left(1+\frac{(C^G+C^B)n^3\log(n)/M}{\omega\left(n(C^G+C^B)+n^3\log(n)/M\right)}\right)\\
&=\scO\left(1+\frac{(C^G+C^B)n^3\log(n)/M}{\omega\left(n^3\log(n)/M\right)}\right)\\
&=\Theta(1)+o(C^G+C^B)~,\\
\end{align*}
which is the claimed amortized time per round. This concludes the proof.
\end{proof}

\section{Supplementary material on the experiments}\label{s:exp-app}

\paragraph{Implementation of \smi.}
As we mention in Section \ref{s:exp}, our variant \ismi{}: 
\textit{(i)} deals with the cases where the input datasets is uniformly random, \textit{(ii)} avoids asking arbitrary queries if more valuable queries are available, and \textit{(iii)} discovers matches during the exploration phase of the algorithm.

To achieve all these goals, we adapted the implementation of \smi{} along different axes. 

First, we combined Phase I and Phase II of \smi{}.
The high-level idea of this modification is to start exploiting immediately the clusters once some clusters are identified, without waiting to estimate all of them. 
We only describe the process of serving recommendations to boys, the process for girls being symmetric.
We maintain for each $b\in B$ a set of girl 
clusters $C_{\textrm{to-ask}}(b)$ for which we do not yet know the preference of $b$, and a set of girl clusters $C_{\textrm{verified}}(b)$ which we already know $b$ likes. Whenever $b$ logs in, if $C_{\textrm{verified}}(b)\not= \emptyset$ we pick a cluster $C\in C_{\textrm{verified}}(b)$ and a girl $g\in C$, and ask $b$ about $g$. If $C_{\textrm{verified}}(b)= \emptyset$ and $C_{\textrm{to-ask}}(b)\not= \emptyset$ we pick a cluster $C\in C_{\textrm{to-ask}}(b)$, ask $b$ his preference for any girl in $C$, remove $C$ from $C_{\textrm{to-ask}}(b)$, update the preference of $b$ for cluster $C$ accordingly, and finally add $C$ into $C_{\textrm{verified}}(b)$ if $b$ likes cluster $C$.
If, on the other hand, $C_{\textrm{verified}}(b)= C_{\textrm{to-ask}}(b)= \emptyset$, and there are no prioritized queries for $b$ (see second modification), we proceed as we would in Phase I of \smi{} (asking $b$ for feedback that helps estimating the clusters).
Whenever the exploration phase discovers a new girl cluster $C$ represented by $g$, we add $C$ into $C_{\textrm{verified}}(b)$ if $\sigma(b,g)=+1$, and into $C_{\textrm{to-ask}}(b)$ if $b$ was not asked about $g$. Whenever a girl $g$ is classified into an existing girl cluster $C$, for the boys $b'$ that provided feedback for $g$ and $C \in C_{\textrm{verified}}(b')$ we remove $C$ from $C_{\textrm{verified}}(b')$ as we now know whether $b'$ likes cluster $C$ or not.

Second, whenever we discover a positive feedback from $b$ to $g$, we prioritize for $g$ the feedback to $b$.
The feedback received by such queries is taken into account when classifying users into clusters.

Third, instead of having Phase II choose girl $g$ arbitrarily (``else" branch in the pseudocode), we let \ismi{} choose girl $g'$ who likes $b$, and if no such $g'$ exists, we select $g''$ for whom we have not yet discovered whether she likes or dislikes $b$. If no such girls exists for $b$, then we serve an arbitrary girl to $b$.

Finally, whenever we compare the feedbacks received by two users, say girl $g$ and $g^r\in G^r$, in order to determine whether $g$ belongs to the cluster of $g^r$, we amended as follows. We insert $g$ into the cluster of $g^r$ by requiring that $\sigma(b,g)=\sigma(b,g^r)$ holds {\em at least} for $\left(|F_{g}\cap F_{g^r}|(1-\frac{1}{\log(n)})\right)$-many boys in $F_{g}\cap F_{g^r}$, in place of {\em all} boys belonging to $F_{g}\cap F_{g^r}$.
This modification aims to cope with the problem of clustering similar users into different clusters due to a very small value in $|F_{g} \triangle F_{g'}|$, that is, the number of boys that like only one out of $g$ and $g'$. In the real-world dataset that we use, we noticed that if we allow no boys to disagree on their feedback to two girls, then the number of girl clusters is almost equal to $|G|$, while allowing a small number of disagreements (that is, a fraction $\frac{1}{\log(n)}$ of the total number of boys) the number of girl clusters reduces drastically. Recall the last six columns of Table \ref{tab:characteristics}.
The same holds for clusters over boys when we consider feedback from girls.

\paragraph{Further experimental results.}
In Table \ref{tab:under-curve} we give the \emph{area under the curve} metric, which sums over time the number of matches that are uncovered at each time-step $t$, divided by the total number of time-steps. This metric captures how quickly, over average, the different algorithms disclose matches. Figure \ref{fig:plots2} contains the plots on the remaining datasets described in Section \ref{s:exp}.

\begin{figure}[t!]
	\begin{center}
		\begin{subfigure}{.5\textwidth}
			\centering
			\includegraphics[trim ={0.5cm 0cm 9.5cm 0},width=\linewidth]{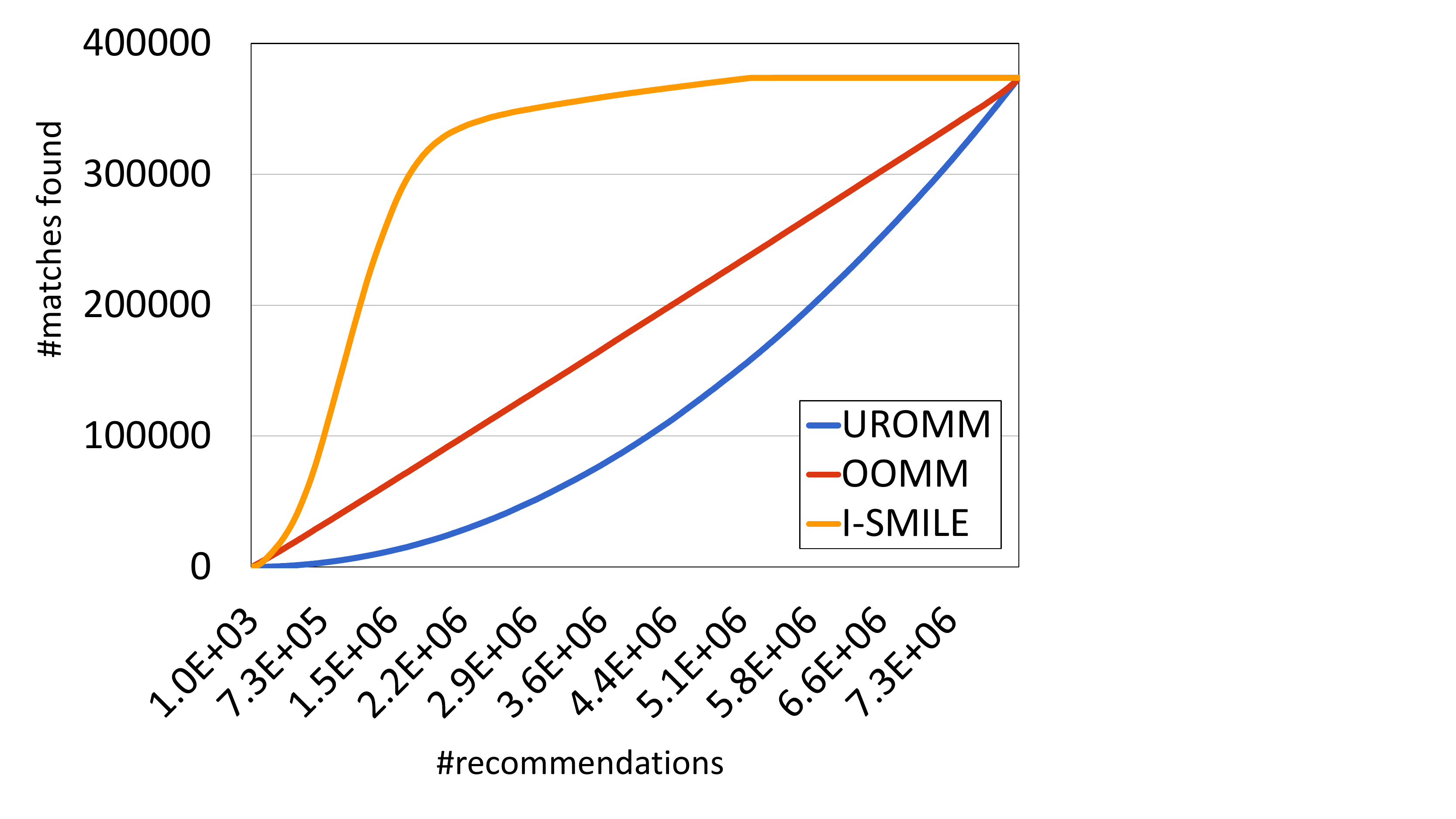}
			\caption{Dataset S-20-23.}
			\label{fig:s-20-23}
		\end{subfigure}%
		\begin{subfigure}{.5\textwidth}
			\centering
			\includegraphics[trim ={0.5cm 0cm 9.5cm 0},width=\linewidth]{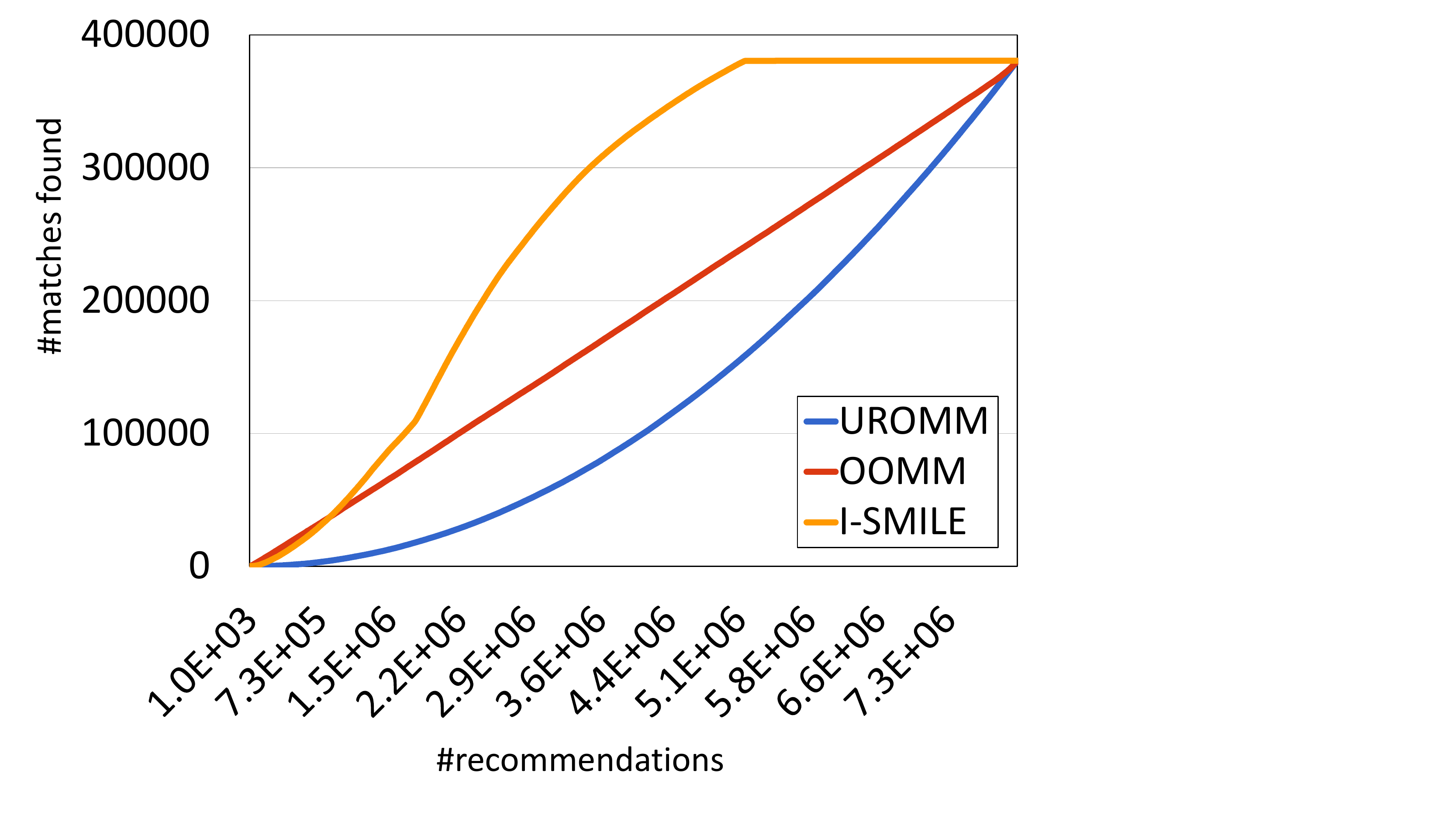}
			\caption{Dataset S-500-480.}
			\label{fig:s-500-480}
		\end{subfigure}
		\begin{subfigure}{.5\textwidth}
			\centering
			\includegraphics[trim ={0.5cm 0cm 9.5cm 0},width=\linewidth]{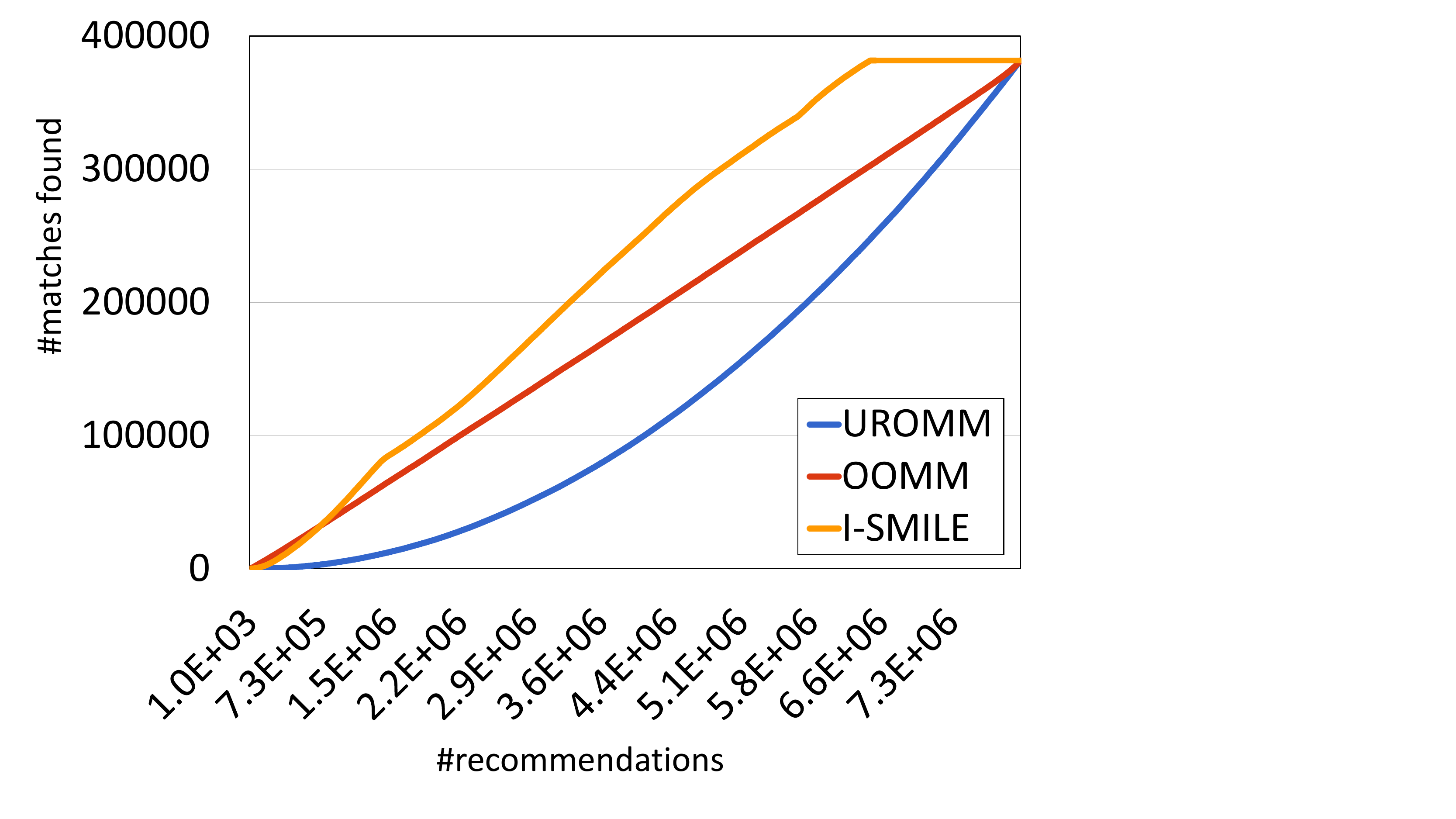}
			\caption{Dataset S-2000-2000.}
			\label{fig:s-uniform}
		\end{subfigure}%
		\begin{subfigure}{.5\textwidth}
			\centering
			\includegraphics[trim ={0.5cm 0cm 9.5cm 0},width=1\linewidth]{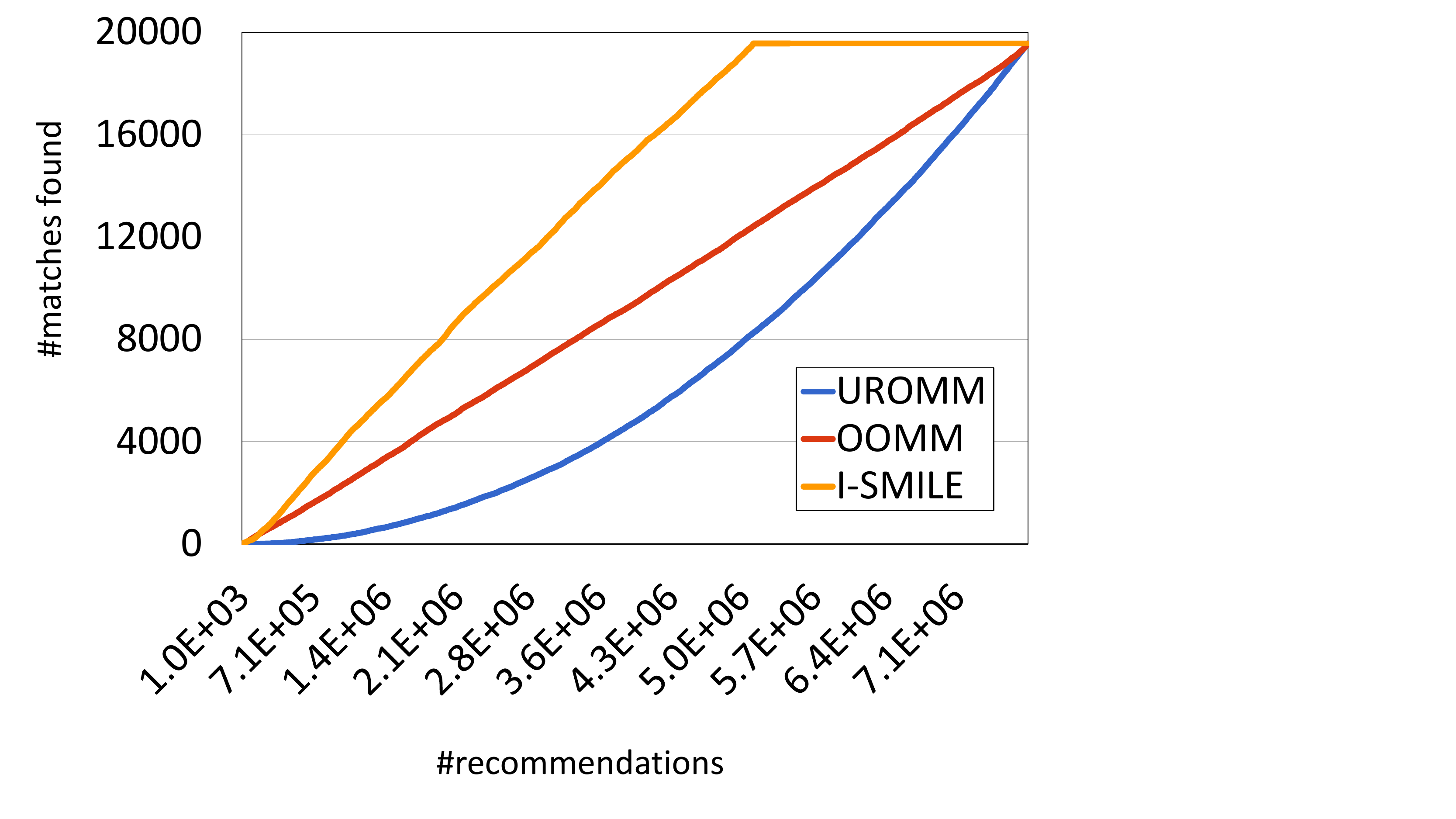}
			\caption{Dataset RW-1007-1286.}
			\label{fig:rw-1007-1286}
		\end{subfigure}
	\end{center}
	\caption {Empirical comparison of the three algorithms \ismi{}, \om, and \uromm\, on the remaining datasets considered in this paper. Each plot reports number of disclosed matches vs. time (no. of recommendations).\label{fig:plots2}}
\end{figure}

\begin{table}[t!]
	\centering
	\renewcommand{\arraystretch}{1.2}
	\small
	\begin{tabular}{|l|rrrrrrr|}
		\hline
		Algorithm & \multicolumn{1}{l|}{S-20-22} & \multicolumn{1}{l|}{S-95-100} & \multicolumn{1}{l|}{S-500-480} & \multicolumn{1}{l|}{S-2000-2000} & \multicolumn{1}{l|}{RW-1007-} & \multicolumn{1}{l|}{RW-1526-} & RW-2265- \\ \hline
		\uromm{}  & $125K$ & $126K$ & $127K$ & $127K$ & $4.69K$ & $6.42K$ &  $8.35K$\\ \hline
		\om{}     & $183K$ & $184K$ & $186K$ & $187K$ & $6.75K$ & $9.55K$ &  $12.21K$\\ \hline
		I-\smi{}  & $312K$ & $296K$ & $263K$ & $225K$ & $9.79K$ & $13.92K$ &  $17.36K$ \\ \hline
	\end{tabular}
	\vspace{.25cm}
	\caption{Area under the curve values of all algorithms running on all datasets in Table \ref{tab:characteristics}.}
	\label{tab:under-curve}
	\vspace{-.5cm}
\end{table}

\end{document}